\documentclass[11pt,twocolumn,a4paper]{article}

\newlength\titlebox 
\usepackage[round]{natbib}

\usepackage[hidelinks,bookmarksopen=true,allcolors=blue,colorlinks=true]{hyperref}

\usepackage{float}

\usepackage{algorithm}
\usepackage{algorithmic}

\usepackage{amsmath}
\usepackage{amsthm}
\usepackage{amssymb}
\usepackage{mathtools}
\usepackage{thmtools} 
\usepackage{thm-restate}

\usepackage[capitalise]{cleveref}
\usepackage{caption}
\usepackage{subcaption}

\usepackage[usenames,dvipsnames]{xcolor}
\usepackage[disable]{todonotes}

\usepackage{bbm}

\usepackage{array}
\newcolumntype{P}[1]{>{\centering\arraybackslash}p{#1}}

\usepackage{upgreek}

\usepackage{thmtools}
\declaretheorem[name=Theorem,refname={Theorem,Theorems},Refname={Theorem,Theorems}]{theorem}
\declaretheorem[name=Lemma,refname={Lemma,Lemmas},Refname={Lemma,Lemmas},sibling=theorem]{lemma}

\declaretheorem[name=Assumption,refname={Assumption,Assumptions},Refname={Assumption,Assumptions}]{assumption}
\declaretheorem[name=Proposition,refname={Proposition,Propositions},Refname={Proposition,Propositions},sibling=theorem]{proposition}

\crefname{assumption}{Assumption}{Assumptions}

\usepackage{pgfplotstable}

\newcommand{\jtodo}[2][]{\todo[color=blue!20,size=\scriptsize,#1]{J: #2}} 

\newcommand{\todoOM}[2][]{\todo[color=red!5,size=\scriptsize,#1]{OM: #2}} 

\newcommand{\E}{\mathbf{E}}

\renewcommand{\P}{\mathbf{P}}
\newcommand{\R}{\mathbb{R}} 
\DeclareMathOperator*{\argmax}{\arg\max}

\usepackage{calc}
\usepackage{accents}

\newcommand{\ca}[1]{\mathcal{#1}}
\newcommand{\T}{^\top}

\newcommand{\info}{\text{I}} 
\newcommand{\iden}{I} 
\newcommand{\indc}{\textbf{1}}
\newcommand{\eqdef}{:=} 
\newcommand{\eqdist}{\overset{\text{d}}{=}} 
\newcommand{\KL}[2]{d_{\text{KL}}(#1\|#2)}
\newcommand{\amaxk}{\text{argmax}^k}
\newcommand{\dr}{\text{d}}

\newcommand{\att}{Y} 
\newcommand{\rew}{R} 
\newcommand{\mrew}{r} 
\newcommand{\attMu}{\nu} 
\newcommand{\priormu}{\mu} 

\newcommand{\reg}{\mathcal{R}} 
\newcommand{\breg}{\mathcal{BR}} 
\newcommand{\Brn}{\text{Bern}} 
\newcommand{\Betad}{\text{Beta}} 
\newcommand{\obs}{\ca{T}} 
\newcommand{\nobs}{T} 
\newcommand{\obsind}{O} 
\newcommand{\act}{A} 
\newcommand{\bact}{\act^*}
\newcommand{\pme}{\mu_0} 
\newcommand{\prsd}{\sigma_0} 
\newcommand{\rsd}{\sigma} 
\newcommand{\hist}{\ca{H}} 
\newcommand{\items}{\mathcal{L}} 
\newcommand{\nitems}{L} 
\newcommand{\listlen}{K} 
\newcommand{\Ngaus}{\mathcal{N}} 
\newcommand{\timH}{T} 
\newcommand{\Upr}{U} 
\newcommand{\linkf}{\mu} 
\newcommand{\cntx}{x} 
\newcommand{\Cntx}{\ca{X}} 

\newcommand{\cmclick}{k} 
\newcommand{\radi}{\beta} 
\newcommand{\radiTwo}{\gamma} 
\newcommand{\cparam}{\theta} 
\newcommand{\cparamSet}{\Theta} 
\newcommand{\paramNorm}{S} 
\newcommand{\infMat}{V} 
\newcommand{\Wrew}{\bar{r}} 
\newcommand{\EvHat}[1]{\ca{E}_{\hat{\cparam},#1}} 
\newcommand{\EvTilde}[1]{\ca{E}_{\tilde{\cparam},#1}} 
\newcommand{\noise}{\eta} 
\newcommand{\minLink}{\kappa} 
\newcommand{\lipchLink}{\ell} 
\newcommand{\infoRatio}{\rho} 
\newcommand{\binE}{d_{\text{KL}}} 
\newcommand{\posMeRew}{\hat{r}} 
\newcommand{\paramSup}{\Theta_0} 

\usepackage{xspace}
\newcommand{\algTS}{{\tt TS}\xspace}
\newcommand{\algBUCB}{{\tt BayesUCB}\xspace}
\newcommand{\algTSB}{{\tt TS-Beta}\xspace}
\newcommand{\algBUCBB}{{\tt BayesUCB-Beta}\xspace}
\newcommand{\algGTS}{{\tt GTS}\xspace}
\newcommand{\algGPTS}{{\tt GTS-P}\xspace}
\newcommand{\algGPTSmean}{{\tt GTS-Pmean}\xspace}
\newcommand{\algTScascade}{{\tt TS-Cascade}\xspace}
\newcommand{\algcascadeKLUCB}{{\tt CascadeKL-UCB}\xspace}
\newcommand{\algcascadeUCBOne}{{\tt CascadeUCB1}\xspace}
\newcommand{\algTopRank}{{\tt TopRank}\xspace}
\newcommand{\algEnsemble}{{\tt Ensemble}\xspace}
\newcommand{\algLinTSLTR}{{\tt LinTS-LTR}\xspace}
\newcommand{\algcascadeLinUCB}{{\tt CascadeLinUCB}\xspace}
\newcommand{\algCLinUCB}{{\tt C-LinUCB}\xspace}
\newcommand{\algcascadeWOFUL}{{\tt CascadeWOFUL}\xspace}
\newcommand{\algCWOFUL}{{\tt C-WOFUL}\xspace}
\newcommand{\algLinTSCas}{{\tt {LinTS-Cascade}}\xspace}
\newcommand{\algLogTSLTR}{{\tt LogTS-LTR}\xspace}

\newcommand{\bndOne}{22 \listlen \sqrt{\nitems \timH}} 

\newcommand{\bndUnderEv}[1][\delta/2]{4 d \listlen \big( \radiTwo_\timH(#1) + \radi_\timH(#1) \big)\sqrt{2\timH d\log(1+\timH/\lambda)}} 

\newcommand{\bndLinTSO}{\tilde{O}\bigg( d^{3/2} \listlen \sqrt{\timH}\bigg)} 
\newcommand{\bndLogTS}{\tilde{O}(\listlen\sqrt{d\nitems\timH})} 

\newcommand{\bndLogTSKappa}{
\listlen \frac{\lipchLink}{\minLink}\big(\radi_t(\delta')+\radiTwo_t(\delta')d\big)\sqrt{2\timH d \log(1+\frac{\timH}{\lambda})}
\\
+
\listlen \frac{2 \radiTwo_t(\delta')}{0.1 \minLink} \sqrt{\frac{8\timH}{\lambda}\log\frac{4}{\delta}}
} 
\newcommand{\bndLogTSKappaOrderwise}{\tilde{O}\bigg(\listlen \frac{\lipchLink}{\minLink} \sigma^2 d\sqrt{\timH d }
\bigg)
}

\newcommand{\LinTSdelta}{\frac{1}{T(\log T+2)}}

\usepackage{eqparbox}

\usepackage{acronym}
\acrodef{TS}{Thompson Sampling}
\acrodef{LTR}{learning to rank}

\title{\bf Overcoming Prior Misspecification in Online Learning to Rank}

\author{
Javad Azizi$\;^1$\hspace{.1in}  Ofer Meshi$\;^2$\hspace{.1in}  Masrour Zoghi$\;^2$\hspace{.1in} Maryam Karimzadehgan$\;^2$ 
\\
\\\centering
$\;^1${\small University of Southern California} $\;^2${\small Google} 
}
\date{}

\begin{document}
\maketitle

\begin{abstract}
  The recent literature on online learning to rank (LTR) has established the utility of prior knowledge to Bayesian ranking bandit algorithms.  
However, a major limitation of existing work is the requirement for the prior used by the algorithm to match the true prior.
In this paper, we propose and analyze adaptive algorithms that address this issue and additionally extend these results to the linear and generalized linear models. We also consider scalar relevance feedback on top of click feedback.
Moreover, we demonstrate the efficacy of our algorithms using both synthetic and real-world experiments.

\end{abstract}

\section{\uppercase{Introduction}}

Learning to rank (LTR) is the problem of ranking a set of items such that the resulting ranked list maximizes a utility function such as user satisfaction.  
This could be the ranking of search results for information retrieval \citep{liu2009learning}, ranking the items in a recommendation system to increase user satisfaction \citep{karatzoglou2013learning,falk2019practical}, or ranking the ads in an ad placement system to enhance user engagement \citep{tagami2013ctr}.    

In \emph{offline} LTR, it is assumed that the \emph{ground truth} utility of the lists has been provided and the goal is to learn a \emph{scoring function} which can be used to rank the items \citep{Situ-Zamani17,mitra2017neural,shen2018multi,meng2020separate}.  
In this setting, it is implicitly assumed that user behavior is time-invariant, however, in many real-world problems user preferences can change dynamically.

To address this issue, \emph{online} LTR algorithms adaptively learn from user feedback.  
Various online learning algorithms have been designed for different user feedback models including the \emph{cascade model} \citep{kveton2015cascading,Zhong-TS_cascade-2021}, the \emph{position-based model} \citep{lagree2016multiple, ermis2020learning}, as well as algorithms designed for multiple user models \citep{zoghi2017online, lattimore2018toprank, li2020bubblerank}.  
In this paper, we mainly focus on cascading bandits and defer extensions to other click models to future work.

Despite being adaptive, online LTR algorithms often suffer from the \emph{cold-start} problem.
More specifically, in the absence of any prior knowledge, the algorithm has to explore aggressively before it can start exploiting the information that it has learned about the users and items.  
The cost of this initial phase of exploration often renders online LTR algorithms impractical:  
this cost is particularly egregious in the ranking setting because of the large action space.  
One natural remedy for this problem is to infuse prior knowledge into the online LTR algorithm, which is the approach adopted by \citet{Bkveton-22-prior}.  

A major limitation of the algorithms and theoretical results in \citet{Bkveton-22-prior} is the assumption that the prior knowledge used by the algorithm does not deviate from the true prior.  
This assumption holds only in the most contrived of situations:  
for instance, for any time-varying data distribution, the prior obtained from previous observations cannot be a perfect prior for the future.  
This paper addresses this shortcoming by devising and analyzing algorithms that can make the most out of imperfect priors by adapting to online data.

More specifically, our contributions are as follows:

\vspace{-3mm}

\begin{itemize}
    \setlength\itemsep{-.3em}
    \item In \cref{sec:TS-LTR}, we propose highly adaptive Gaussian \ac{TS} algorithms for LTR in non-contextual and contextual settings.
    We consider the non-contextual setting in \cref{sec:GTS}, the linear model in \cref{sec:lin}, and the generalized linear model (GLM) in \cref{sec:TS-GLM}. 
    
    \item Our linear model in \cref{sec:lin} handles scalar relevance feedback, generalizing our framework beyond click (binary) feedback.

    \item We derive Bayes regret bounds for our algorithms in \cref{thm:GaussTS,thm:LinTS,thm:LogLTR1,thm:Log-LTR2}.  
    In the non-contextual setting, our bound is near-optimal.  
    In the contextual settings, these bounds are the first Bayesian bounds in the literature to the best of our knowledge.

    \item We conduct both synthetic and real-world experiments in \cref{sec:exp}. %
    The synthetic experiments demonstrate that even though our algorithms start with an imperfect prior (e.g., flat), they quickly adapt to the environment and achieve competitive results compared to existing approaches. %
    In particular, in the presence of prior misspecification, the performance of our algorithms does not deteriorate as severely as existing online LTR algorithms which use prior information.
    We conduct our real-world experiments on publicly available learning to rank datasets to test the efficacy of our algorithms in more realistic settings.
\end{itemize}

\section{\uppercase{Setting}}\label{sec:setting}

We use the notation $[T]\eqdef \{1,\cdots,T\}$ for any integer $T$. For any vector or set $v$, let $v(i)$ be its $i$'th element. 

We consider an online LTR problem where $\items$ is the set of items to choose from, with size $\nitems$. The
agent interacts with the environment, such as users in a recommender system, over $T$ rounds. At round $t\in[T]$, 
the agent displays a ranked list $\act_t$ of $\listlen\ll \nitems$ items to the user, i.e., $\act_t\in\Pi_{\listlen}(\items)$, where $\Pi_{\listlen}(\items)$ is the set of all tuples of ${\listlen}$ distinct items out of $\items$. The reward feedback depends on the attractiveness of items recommended to the user.  
We denote by $\att_{i,t}$ the attractiveness of item  $i\in\items$ at round $t$, which is an independent Bernoulli random variable with mean $\attMu_{i,t}$, i.e. $\att_{i,t}\sim \Brn(\attMu_{i,t})$. 

Let $\att_{t} \eqdef \att_{\act_t,t} \eqdef (\att_{\act_t(k),t})_{k=1}^{\listlen}$ be the vector of attraction indicators (feedback) at round $t$, and its corresponding reward be $\rew(\att_{\act_t,t})$.  
In the \emph{cascade model} (CM) we assume the user examines the first position in $\act_t$ with probability 1. If position $k\in[\listlen]$ is examined and the item at that position is attractive, i.e., $\att_{\act_t(k), t} = 1$, the user clicks on it and does not examine the rest.  
If the user does not click on the item at position $k$, they examine the next position, $k + 1$.  
In this model, the reward at round $t$ is 1 if there is any attractive item in $\act_t$ and 0 otherwise: $\rew(\att_{\act,t})=1-\left(\prod_{k=1}^{\listlen}(1-\att_{\act_t(k),t})\right)$.
Due to the independence of the attractiveness of items the expected reward is
\begin{align}
    \mrew_{\act,t} & \eqdef \E[\rew(\att_{\act,t})] = 1-\prod_{k=1}^{\listlen}(1-\attMu_{\act(k),t}) 
    \label{eq:Cascade-Model}
    \;.
\end{align}
We set $\cmclick_t=\min\{k\in[\listlen]:\att_{\act_t(k),t}=1\}$ to be the click position, and use the convention $\min\emptyset=\listlen$.  
That is if the user does not click on any of the items then $\cmclick_t=\listlen$.  
Then $\obsind_{i,t}=\sum_{k=1}^{\cmclick_t}\indc(i=\act_t(k))$ is the indicator of item $i$ being examined (observed) at round $t$.  
The set of rounds where item $i$ is observed by round $t$ is $\obs_{i,t}\eqdef\{s\le t: i\in\act_s,i\leq \cmclick_s\}$, with size $\nobs_{i,t} \eqdef |\obs_{i,t}|$.

Let $\bact_t=\argmax_{\act\in\Pi_k({\nitems})}\mrew_{\act,t}$ be the best ordered list (best action) at round $t$. Then the \emph{frequentist regret} at round $t$ is $\reg_t \eqdef \mrew_{\bact_t}-\mrew_{\act_t}$, and the cumulative (frequentist) regret for horizon $\timH$ is
$\reg(\timH) \eqdef \E[\sum_{t=1}^\timH \reg_t]$. We define the \emph{Bayes regret} as the expectation of the frequentist regret over the problem instances of the attractiveness probabilities $\attMu$:
\begin{equation*}
    \breg(\timH) \eqdef \E[\reg(\timH)]=\E\bigg[\sum_{t=1}^\timH \E_t[\reg_t]\bigg]\;,
\end{equation*}
where $\E_t[\cdot]\eqdef \E[\cdot | \hist_t]$ is the conditional expectation given $\hist_t$, the trajectory of clicks and actions up to but not including round $t$. The second equality above holds by the tower rule.

\section{\uppercase{Thompson Sampling for LTR}}\label{sec:TS-LTR}

In this section, we develop \ac{TS} algorithms for LTR under various parameterizations, including the non-contextual setting as in \cref{sec:setting}, and also its extension to linear and generalized linear contextual bandits in \cref{sec:lin,sec:TS-GLM}. We assume the prior is such that $\attMu\sim P_0(\priormu)$, where $\priormu$ is the prior parameter.
We denote the posterior at round $t$ with  $P_t(\hat{\attMu}_t)$, where $\hat{\attMu}_t$ is the posterior estimate of $\attMu$ given $\hist_t$.

\subsection{Gaussian Thompson Sampling}\label{sec:GTS}

\jtodo{changed a little to clarify the Gaussian Bernoulli confusion, reviewer 4}
In this section, we propose an adaptive Gaussian \ac{TS} algorithm for the LTR problem under the cascade model. \citet{agrawal2013further} showed that Gaussian \ac{TS} performs very well when applied to bandits with a bounded distribution. Inspired by this observation, we propose a Gaussian \ac{TS} algorithm for our LTR problems which has Bernoulli (so bounded) rewards, i.e., implicitly assume a Gaussian feedback while it is Bernoulli. See Gaussian \ac{TS} paragraph of \cref{sec:related-works} for further details.

The pseudo-code of our algorithm is in \cref{alg:Gauss-LTR}. 
We begin with a Gaussian prior $P_0(\priormu)=\Ngaus(\pme, \prsd)$,
with mean $\pme$ and standard deviation $\prsd$ (Line~\ref{alg:GTS:init}).  
The posterior updates are as follows (Line~\ref{alg:GTS:post}):
\begin{align}
    \att_{i,t}~ | ~\attMu_i &\sim \Ngaus(\attMu_i, \rsd^2) & \text{(Likelihood)}\nonumber
    \\
    \attMu_{i,t} \sim  P_{i,t}(\attMu_i|\hist_t) &= \Ngaus(\hat{\attMu}_{i,t}, \hat{\sigma}_{i,t}^2) & \text{(Posterior)}
    \label{eq:posterior-GTS}
    \;,
\end{align}\todoOM{(and compare to Beta TS)}
where by \citet[Eqs 20 and 24]{murphy2007conjugate} 
\begin{align*}
    \hat{\attMu}_{i,t} = \textstyle\left(\frac{\sum_{\ell=1}^t \att_{i,\ell}\obsind_{i,t}}{\rsd^2}+\frac{\pme}{\prsd^2}\right)\hat{\sigma}_{i,t}^{2}
    \textup{ and }
    \hat{\sigma}_{i,t}^2 = \left(\frac{1}{\prsd^2}+\frac{\nobs_{i,t}}{\rsd^2}\right)^{-1}
\end{align*}
Note that the posterior at round $t+1$ matches the one computed in \cref{eq:posterior-GTS} in a Gaussian \ac{TS} (see \citet{murphy2007conjugate} for detail). \jtodo{added this} We compare this to Beta \ac{TS} \citep{Bkveton-22-prior} in \cref{app:Gaus-vs-Bern}, where we highlight its advantages over the Beta \ac{TS}.

At round $t$, our algorithm computes the posterior based on $\hist_t$ as given in \cref{eq:posterior-GTS}.  
Next, it acquires a sample from the posterior distribution for each item and recommends the items corresponding to the top $\listlen$ samples (Lines~\ref{alg:GTS:sample} and \ref{alg:GTS:recom}).  
Finally, it updates $\hist_{t+1}$ by including the action and reward pair (Line~\ref{alg:GTS:observ}).

We analyze this algorithm with the default prior $\Ngaus(0,1)$. In \cref{sec:exp-stand}, we show empirically that the effect of this prior vanishes quickly. However, in a Bernoulli environment like our non-contextual model $\Ngaus(0,0.25^2)$ might be more favorable given that Bernoulli random variables are $0.25$-subgaussian.
\begin{algorithm}[tb]
   \caption{TS for LTR}
    \label{alg:Gauss-LTR}
\begin{algorithmic}[1]
    \STATE {\bf Initialize:} $\hist_1=\emptyset$ and set prior $P_0 := \Ngaus(\pme, \prsd)$ \label{alg:GTS:init}
    \FOR{$t=1,\dots,\timH$}{
        \STATE Compute posterior $P_{i,t}(\attMu_i|\hist_t)$ based on \cref{eq:posterior-GTS} \label{alg:GTS:post}
        \STATE Sample $\tilde{\attMu}_{i,t} \sim P_{i,t}(\attMu|\hist_t)$ for all $i\in\items$ \label{alg:GTS:sample}
        \COMMENT{{Posterior sample}}
        \STATE $\act_t \in \argmax_{\substack{ \act\subset\items\\|\act|={\listlen}}}\sum_{i\in \act} \tilde{\attMu}_{i,t} \label{alg:GTS:recom} $\COMMENT{{ Recommend the top posterior samples}}
        \STATE Observe $\att_t$ and let $\hist_{t+1}\gets \hist_{t}\cup\{(\act_t,\att_t)\}$ \label{alg:GTS:observ}
    }
    \ENDFOR
\end{algorithmic}
\end{algorithm}

\paragraph{\cref{alg:Gauss-LTR} Analysis:} 
We follow a similar line of reasoning as \citet[Theorem 1]{bubeck2013prior} for Gaussian \ac{TS} Bayes regret. The challenge is to apply that result to the cascading bandits. We overcome this using a results about the cascade model shown in \citet[Lemma 7]{Bkveton-22-prior}, which bounds $\reg_t$ with $\sum_{k\in\act_t}\mrew_{\bact_t(k), t}-\mrew_{k,t}$. 
\begin{restatable}[Gaussian TS]{theorem}{GaussTS}
\label{thm:GaussTS}
 The Bayes regret of \cref{alg:Gauss-LTR} with prior $\Ngaus(0,1)$ satisfies
 \begin{equation*}
     \breg(\timH) \le \bndOne\;.
 \end{equation*}
\end{restatable}
The proof is in \cref{app:GTS-proofs}. Compared to the result of \citet{Bkveton-22-prior}, it appears that we need to pay a $O(\sqrt{K})$ price for not knowing the prior. If we use the true prior we can shave off this extra term by using exact confidence bounds as used in \citet{Bkveton-22-prior}. In terms of $\nitems$ and $\timH$, this bound is of optimal order since they match the lower bound for Gaussian \ac{TS} \citep{agrawal2013further}. In addition, the constants are much smaller compared to the frequentist bounds of \citet{Zhong-TS_cascade-2021} (see \cref{sec:related-works}). Recently, \citet{vial2022minimax} proved a matching lower bound of $\Omega(\sqrt{\nitems \timH})$, which shows that our algorithm achieves near-optimal regret bound.

In the next two sections, we deal with the contextual setting. In this setting, we use a feature vector $\cntx_{i,t}\in\Cntx_t\subset\R^d$ for each item $i$ at round $t$, where $\Cntx_t$ is the set of feature vectors for the items to be ranked at round $t$.  
For instance, in an information retrieval system, this vector encodes user, query, and item information.  
We use the notation $\cntx_t := (\cntx_{i,t})_{i\in\items}$ for the array of feature vectors in the list.

\subsection{Linear Contextual LTR}\label{sec:lin}

In this setting, we assume that the reward generated for each item is governed by a linear function.
This setting is relevant to applications where the feedback ($\att$) is not binary (not Bernoulli clicks) but a real number such as the amount of time the user spent on a page or their level of satisfaction with the page, etc.
We continue to assume a cascade user model in the sense that the user scans down the list and either generates no reward on any item in the list or generates independent rewards for some items and then abandons the list at one item.  

The reward at round $t$ for action $\act$ is the sum of item level rewards of the observed items, i.e.,
\begin{align*}
    \rew(\att_{\act,t})=\sum_{k=1}^\listlen \obsind_{\act(k),t} \att_{\act(k),t} \;.
\end{align*}
Recall from \cref{sec:setting} that $\obsind_{i,t}$ is the independent random variable that indicates whether item $i$ was observed at round $t$. By independence we have $\mrew_{\act,t} = \sum_{k=1}^\listlen \mrew_{\act(k),t}$.  
Note that this reward is the same as the reward in the case of the cascade model if $\att$ is binary since the right-hand side of \cref{eq:Cascade-Model} is 0 when no item in the list is clicked and 1 when there is a click. Also, in the cascade model there is at most one click.  

We begin by laying out our assumptions in this setting.
\begin{assumption}[Reward]\label{asm:noise}
     If we have $\obsind_{i,t}=1$, then the feedback for item $t$ is $\att_{i,t} = \cntx_{i,t}\T \cparam +\noise_t$, where $\cparam$ is the unknown parameter of the problem and $\noise_t$ is $\sigma^2$-sub-Gaussian. Thus, $\mrew_{i,t}=\cntx_{i,t}\T\cparam$ is the mean reward (feedback) of item $i$ at round $t$ given feature $\cntx_{i,t}$.
\end{assumption}
\begin{assumption}[Features]\label{asm:cntx}
    We assume $\| \cntx_{i,t}\|\le 1$ for all $i,t$, and that the set of feature vectors $\Cntx\subset\R^d$ is compact.%
    \footnote{We use $\|\cdot\|$ to denote the $\ell_2$ norm.}
\end{assumption}
\begin{assumption}[Parameters]\label{asm:param}
     We assume $\|\cparam\|\le \paramNorm$ for some $\paramNorm \in\R^+$. 
\end{assumption}
\cref{alg:LinGauss-LTR} contains the pseudo-code of our linear \ac{TS} algorithm.
In \cref{alg:LinGauss-LTR} we have used $\radi_t(\delta_t) \eqdef \sigma^2 \sqrt{2 \log \frac{(\lambda + t)^{d/2} \lambda^{-d/2}}{\delta_t}} + \sqrt{\lambda}\paramNorm$, where $\delta_t=\delta/2^{\max(1,\lceil\log(t)\rceil)}$ and $\delta=\LinTSdelta$.
This algorithm is similar to \cref{alg:Gauss-LTR}, with the main difference being the calculation of the posterior sample. In particular, we use \emph{regularized least-squares} estimate of $\cparam$ (Line~\ref{alg:lin:estimate}) and sample $\tilde{\cparam}_t$ from the posterior (Line~\ref{alg:lin:sample}). Then, we generalize the posterior sample for item $i$ with $\cntx_{i,t}\T\tilde{\cparam}_t$ in Line~\ref{alg:lin:recom}. 

\begin{algorithm}[tb]
   \caption{Linear \ac{TS} for LTR}
    \label{alg:LinGauss-LTR}
\begin{algorithmic}[1]
    \STATE {\bfseries Input:} The regularization parameter $\lambda\in\R^+$, $\paramNorm$
    \STATE {\bf Initialize:} $\infMat_1\gets\lambda\iden_d$, $\Wrew_1\gets 0$
    \FOR{$t=1,\dots,\timH$}{
        \STATE Receive the current context $\cntx_t$
        \STATE Let $\hat{\cparam}_t=\infMat_t^{-1} \Wrew_t$. \label{alg:lin:estimate} {\COMMENT{{ Compute the posterior parameter}}}
        \STATE Sample 
        $\tilde{\cparam}_t \sim \Ngaus( \hat{\cparam}_t, \radi^2_t(\delta_t)\infMat^{-1}_t)$ \label{alg:lin:sample} {\COMMENT{{Posterior sample}}}
        \STATE Recommend $\act_t \in \argmax_{\substack{ \act\subset\items\\|\act|={\listlen}}} \sum_{i\in \act}  \cntx_{i,t}\T\tilde{\cparam}_t $ \label{alg:lin:recom} {\COMMENT{{ Recommend the top posterior samples}}}
        \STATE Observe $\att_t$ %
        \STATE $\infMat_{t+1}\gets \infMat_t+\sum_{i\in\act_t}\obsind_{i,t}\cntx_{i,t}\cntx_{i,t}\T$ 
        \STATE $\Wrew_{t+1} \gets \Wrew_t +\sum_{i\in\act_t}\obsind_{i,t}\cntx_{i,t}\att_{\act_t(i),t}$
    }
    \ENDFOR
\end{algorithmic}
\end{algorithm}

\paragraph{\cref{alg:LinGauss-LTR} Analysis:} We derive the first Bayesian regret analysis for linear \ac{TS} for cascading bandits. Our innovation lies in providing a Bayes regret by converting the frequentist analysis in prior work on linear \ac{TS} \citep{abeille2017linear} to the cascade model and bound the Bayes regret.
\begin{restatable}[\cref{alg:LinGauss-LTR} Regret]{theorem}{LinLTR}
\label{thm:LinTS}
    Under \cref{asm:cntx,asm:param,asm:noise}, the Bayes regret of \cref{alg:LinGauss-LTR} is
    $\breg(\timH)=\bndLinTSO$, where $\tilde{O}$ hides logarithmic terms.
\end{restatable}%
Note that compared to \cref{thm:GaussTS}, the dependence on $L$ is replaced with a $d$ dependence.\footnote{The $O$ operator does not hide any $\timH$ term.}
Compared to frequentist upper bound in \citet{vial2022minimax}, i.e., $\tilde{O}(\sqrt{\timH d(d+\listlen)})$, our regret bound is order-wise sub-optimal, but it does not include huge constants like theirs.
To the best of our knowledge, no lower bound is known for cascading bandits in the linear case.
We prove \cref{thm:LinTS} in \cref{app:LinTS-proofs} starting with a decomposition that separates the posterior sample error and the estimation error.  
Then, we bound each term using ellipsoidal high probability bounds from \citet{abbasi2011improved,abeille2017linear}.

\subsection{Generalized Linear Contextual LTR}\label{sec:TS-GLM}

In a click model, it is natural to consider the clicks as Bernoulli random variables as in \cref{sec:setting}.  
In this section, we employ a \emph{generalized linear model} (e.g. \emph{logistic} model) to model the Bernoulli attractiveness random variables in a contextual setting. 

We model the attractiveness as $\att_{i,t} \sim \Brn(\linkf(\cntx_{i,t}\T \cparam))$ where $\linkf(x)$ is the so-called \emph{link} function, e.g., $\linkf(x)=1/(1+\exp(-x))$ is the Sigmoid function. The reward functions, $\rew$ and $\mrew$, are defined the same as in \cref{eq:Cascade-Model} where $\attMu$ is replaced with $\linkf(\cntx\T\cparam)$, which is the probability of attractiveness for item $i$ at round $t$. We have the following assumption on the reward noise.
\begin{assumption}[Reward]\label{asm:noise-GLM}
     We assume that if $\obsind_{i,t}=1$ then the feedback for item $t$ is $\att_{i,t} = \linkf(\cntx_{i,t}\T \cparam) +\noise_t$, where $\noise_t$ is $\sigma^2$-sub-Gaussian.
\end{assumption}

The pseudo-code of our algorithm for this setting is in \cref{alg:Log-LTR}. We define $\minLink \eqdef \inf_{\cparam\in\cparamSet, \cntx\in\Cntx}
\dot{\linkf}(\cntx\T\cparam)$ where $\dot{\linkf}(z)=\frac{\partial \linkf(z)}{\partial z}$. 
The general structure of the algorithm is the same as before. Therefore, we discuss only the new parts. At round $t$, we estimate $\cparam$ (Line~\ref{alg:GLM:est}) by solving the following equation
\begin{equation}
    \sum_{s=1}^t\sum_{k=1}^\listlen \obsind_{k,t} \left(Y_{\act_s(k),s}-\linkf(\cntx_{\act_s(k),s}\T\hat{\cparam}_t)\right)\cntx_{\act_s(k),s}=0\label{eq:Log-Param-estimate}\;,
\end{equation}
for $\hat{\cparam}_t$, using a \emph{iteratively reweighted least
square} (IRLS) oracle \citep{green1984iteratively}.  
Next, we sample from the Laplace approximation of the posterior \citep[Chapter 4.5.1]{bishop2006pattern} in Line~\ref{alg:GLM:sample}. We provide an alternative algorithm in \cref{sec:log-alter} which takes single Newton steps \citep{gentile2012multilabel} and therefore it is more computationally attractive compared to using IRLS. 

\begin{algorithm}[tb]
   \caption{GLM \ac{TS} for LTR}
    \label{alg:Log-LTR}
\begin{algorithmic}[1]
    \STATE {\bfseries Input:} The regularization parameter $\lambda\in\R^+$, $\paramNorm$, and ${\listlen}$
    \STATE {\bf Initialize:} $\infMat_1\gets \lambda\iden_d$.%
    \FOR{$t=1,\dots,\timH$}{
        \STATE Receive the current context $\cntx_t$.
        \STATE Solve \cref{eq:Log-Param-estimate} for $\hat{\cparam}_t$ \label{alg:GLM:est} \COMMENT{{ Compute the posterior parameter}}
        \STATE Sample $\tilde{\cparam}_{t} \sim \Ngaus(\hat{\cparam}_t, \radi_t^2(\delta_t) \infMat_t^{-1}/\minLink^2)$ \label{alg:GLM:sample} \COMMENT{Laplace posterior sample}
        \STATE $\act_t \in \argmax_{\substack{ a\subset\items\\|a|={\listlen}}}\sum_{i\in a} \cntx_{i,t}\T\tilde{\cparam}_t $ \label{alg:GLM:recom} \COMMENT{{ Recommend the top posterior samples} }
        \STATE Observe $\att_t$ %
        \STATE $\infMat_{t+1}\gets \infMat_t + \sum_{i\in\act_t}\obsind_{i,t}\linkf(\hat{\cparam}_t\T\cntx_{i,t})\big(1-\linkf(\hat{\cparam}_t\T\cntx_{i,t})\big)\cntx_{i,t}\cntx_{i,t}\T$
    }
    \ENDFOR
\end{algorithmic}
\end{algorithm}

\paragraph{\cref{alg:Log-LTR} Analysis:} We add the following assumption on the properties of $\linkf$, and derive a regret bound independent of $\nitems$ as follows.
\begin{assumption}[Link Function]\label{asm:linkFunction}
     The link function, $\linkf:\R\mapsto\R$, is continuously differentiable, Lipschitz with constant $\lipchLink$ and $\minLink > 0$.
\end{assumption}
For example, $\lipchLink=1/4$ for the Sigmoid function, and $\minLink$ depends on $\sup_{\cparam\in\cparamSet, \cntx\in\Cntx}
\cntx\T\cparam$, which by \cref{asm:cntx,asm:param} is at most $\paramNorm$.

\begin{restatable}[\cref{alg:Log-LTR} Regret]{theorem}{LogLTRTwo}
\label{thm:Log-LTR2}
   The Bayes regret of \cref{alg:Log-LTR} under \cref{asm:cntx,asm:param,asm:noise-GLM} satisfies
   \begin{align*}
       \breg(\timH)
        \le \bndLogTSKappa
   \end{align*}
   with probability at least $1-\delta$ where $\delta'=\tfrac{\delta}{4\timH}$. Here $\radiTwo_t(\delta)=\radi_t(\delta)\sqrt{c d \log(c' d/\delta)}$ with $c,c'$ constants such that $\P\left(\|\att_{i,t}-\cntx_{i,t}\T\cparam\|\le \sqrt{c d\log(c' d/\delta)}\right)\geq 1-\delta$. We also show by choosing a proper $\delta$ and $\gamma$ the regret is 
   $ \bndLogTSKappaOrderwise
       $.
\end{restatable}

We prove this in \cref{app:log-proofs}, using an extension of the linear \ac{TS} regret bounds and the definition of $\minLink$ and $\lipchLink$. 
To the best of our knowledge, there is no prior work on cascading bandits with comparable Bayes regret bound. \citet{santara22a} developed an algorithm for a routing problem in a cascading bandit setting. Their regret bound is $O\Big(d\log(1+\timH)\sqrt{\timH e^{2 \paramNorm} (\listlen+d\log(1+\tfrac{\listlen d \timH}{\delta})}\Big)=\tilde{O}\big(d \timH e^ \paramNorm \sqrt{ \timH(\listlen+d)}\big)$ for $\delta\in(0,1/e]$ with probability $1-\delta$. This is better than our bound by a factor of $\sqrt{\listlen}$, but it has an exponential dependence on $\paramNorm$ which could be as bad as $1/\minLink$ (see Remark 1 in \citet{santara22a}).

In prior works on GLM bandits, $1/\minLink$ usually appears in the regret bounds. However, depending on $\Cntx$, this term could be arbitrarily large. As such, it is desirable if our regret analysis does not contain $\minLink$. We employ an information-theoretic approach as in \citet{neu2022lifting} to derive the following result for our LTR setting. We define a new information ratio and bound it for each item in the recommendation list, which readily gives an upper bound on the Bayes regret of our algorithm. We focus on the Sigmoid function for the simplicity of the proofs.
\begin{restatable}[\cref{alg:Log-LTR} Regret without $\minLink$]{theorem}{LogLTROne}
\label{thm:LogLTR1}
   If \cref{asm:cntx,asm:param,asm:noise-GLM} hold and $\linkf$ is the Sigmoid function, then the Bayes regret of \cref{alg:Log-LTR} satisfies
   \begin{align*}
       \breg(\timH)
        \le \listlen \sqrt{2 \nitems (d\log(2\paramNorm \timH +1) +1)} =\bndLogTS
   \end{align*}
\end{restatable}
Our information-theoretic proof of this result is outlined in \cref{app:log-proofs-noKappa}. We note that the dependence on $\nitems$ is not desirable in a contextual setting, but this bound is based on the state-of-the-art results for GLM bandits as it avoids $\minLink$.

\section{\uppercase{Related Work}}\label{sec:related-works}

\paragraph{Cascading bandits:} The cascade click model was
introduced in \citet{richardson2007predicting,craswell2008experimental} and applied to online LTR in \citet{kveton2015cascading}. 
\citet{Zhong-TS_cascade-2021} proposed \algTScascade for non-contextual cascading bandits which is a semi-Thompson sampling (TS) algorithm (dependent posterior samples used) with Gaussian updates which start with an uninformative Gaussian prior (flat prior). They present a frequentist regret upper bound for this algorithm, however, their regret bound includes a huge constant (Lemma 4.3 and the last equation in Section 4 therein).
Our \cref{alg:Gauss-LTR} improves on their algorithm by applying an exact \ac{TS} Gaussain update. Our analysis also improves the regret bound, although it is a Bayes regret. 

\paragraph{Gaussian TS:}
\citet{agrawal2013further} showed the problem-independent regret bound of Gaussian \ac{TS} for a $[0,1]$ reward ${\listlen}$-armed bandit problem is $O(\sqrt{{\listlen} \timH \log {\listlen}})$ for $\timH>{\listlen}$. \citet{bubeck2013prior} used refined confidence bounds and improved this bound by removing the $\log {\listlen}$ factor.  \citet{abeille2017linear} revisited linear \ac{TS} for contextual bandits under the frequentist setting.  %
In a related work, \citet{liu2022gaussianImagination} quantified the amount of miss-specification in Gaussian \ac{TS} applied to Bernoulli bandits. LTR click models usually assume Bernoulli click random variables which is $[0,1]$-bounded. Similarly, we introduce a Gaussian \ac{TS} for our LTR problem.
We propose \cref{alg:LinGauss-LTR} and the first Bayesian regret analysis for Linear \ac{TS} in the cascading model. \citet{Zhong-TS_cascade-2021} proposed \algLinTSCas for linear cascade model and showed $O(\min\{\sqrt{d}, \sqrt{\log L}\} \listlen d \sqrt{\timH} (\log \timH)^{3/2})$ frequentist regret where the constant is $(1+\lambda)(1+8\sqrt{\pi} e^{7/(2\lambda^2)})$ which could be very large. \citet{faury2020improved-logistic} proposed the state-of-the-art UCB algorithm for logistic bandits which achieves $O(\sqrt{\minLink \timH} d \log \timH)$ regret bound. This result improves on the classic result of \citet{filippi2010parametric} removing a $\sqrt{\minLink}$ factor. Depending on $\Cntx$, this number could be arbitrarily large. \citet{neu2022lifting} provide a regret bound for \ac{TS} over logistic bandits which is $\tilde{O}(\sqrt{d \nitems \timH})$ for $\nitems$ arms. This bound is $\sqrt{\nitems}$ worse than the above but it is $\minLink$ free which makes it viable in that sense.

\paragraph{Related algorithms:} The state-of-the-art prior works comparable to our work include (i) prior-free online LTR algorithms like \algcascadeKLUCB, \algcascadeUCBOne from \citet{kveton2015cascading}, and \algTopRank \citep{lattimore2018toprank}, (ii) the only prior-based adaptive algorithms, \algBUCB and \algTS \citep{Bkveton-22-prior}, and finally, (3) non-adaptive prior-based algorithms where the items are ranked according to their maximum-a-posteriori (MAP) attraction probabilities estimated from the same prior
as in \algBUCB and \algTS, such as \algEnsemble in \citep{Bkveton-22-prior}.

\algcascadeLinUCB is a variant of the LinUCB algorithm applied to the linear cascade bandits setting by \citep{zong2016cascading}. Note that {\tt CascadeLinTS} from that paper is almost the same as our \algLinTSLTR, except the calculation of the variance of posterior. However, they do not have an analysis for this algorithm. In the linear setting, however, the state-of-the-art is \algcascadeWOFUL \citep{vial2022minimax} which uses Bernstein and Chernoff confidence bounds to develop a variance-aware algorithm. \algLinTSCas \citep{Zhong-TS_cascade-2021} is again similar to our \algLinTSLTR but with different variance calculations and it only has frequentist analysis with huge constants.

There are several other click models considered in the bandit literature. Recently, \citet{ermis2020learning-PBM} proposed a linear UCB and a linear \ac{TS} algorithm for the \emph{position based} click model (PBM). They, however, do not analyze the regret of these algorithms.

\paragraph{GLM bandits:} Several recent works study GLM and especially logistic bandits. \citet{dong2019performance} analyzed the performance of \ac{TS} for logistic bandits using an information-theoretic approach which results in a regret bound without $\minLink$ in it but with other potentially large constants. \citet{neu2022lifting} improved on this and proposed a regret bound free of large constants. In \cref{thm:LogLTR1} we take a similar approach to this. \citet{dumitrascu2018pg} proposed an improvement over the Laplace estimate of the posterior for logistic \ac{TS} but did not bound its regret. Several other recent works including \citet{faury2020improved-logistic,abeille2021instance,faury2022jointly}  study UCB and optimism in face of uncertainty (OFU) based algorithms which we could further consider for future work.

\paragraph{Offline initialization:} Considering the cold-start issues, we can learn the prior information from the historical data and use it in near-optimal algorithms like in \citet{Bkveton-22-prior}. The Web30K experiment in \citet{Bkveton-22-prior} practically does this.

\begin{figure*}[t]
\centering
    \begin{minipage}[t]{.32\textwidth}
    \includegraphics[width=\textwidth]{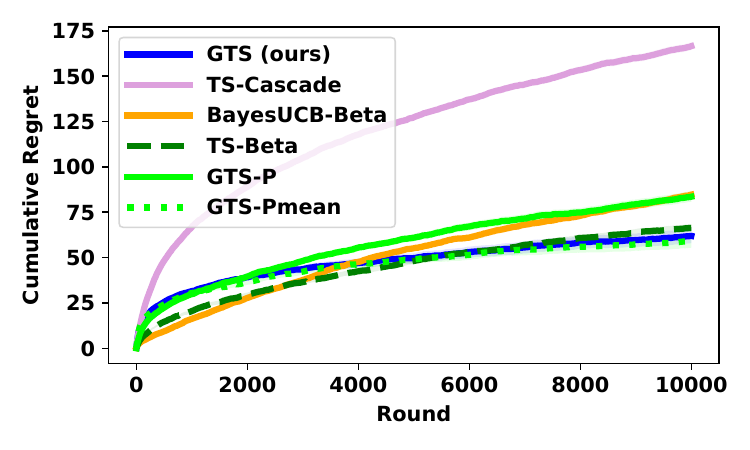}
    \caption{Synthetic non-contextual setting.}
     \label{fig:stand}%
    \end{minipage}
\hfill
     \begin{minipage}[t]{.32\textwidth}
    \includegraphics[width=\textwidth]{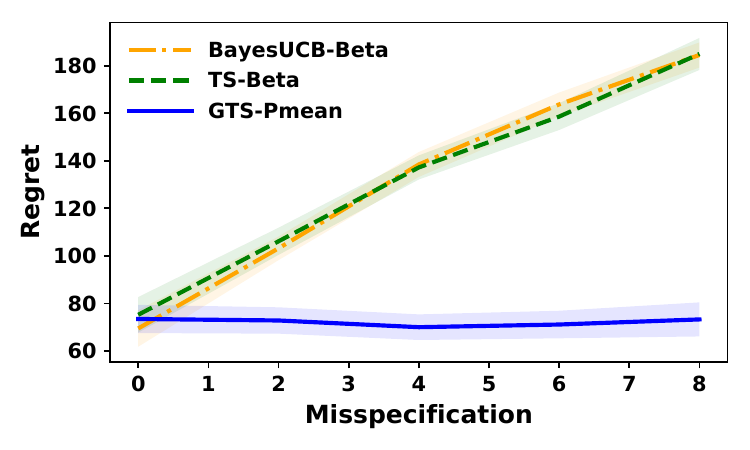}
    \caption{Synthetic non-contextual setting with prior misspecification.}
    \label{fig:stand-misspec} %
    \end{minipage}
\hfill
    \begin{minipage}[t]{.32\textwidth}
    \includegraphics[width=\textwidth]{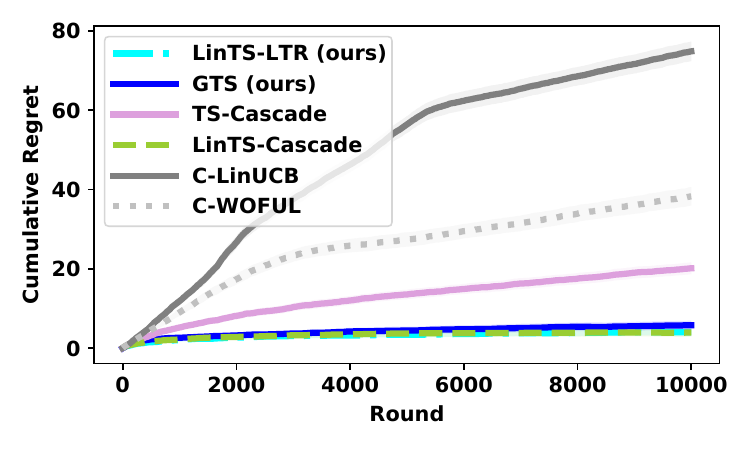}
    \caption{Linear setting experiment.}
    \label{fig:lin} %
    \end{minipage}
\end{figure*}

\section{\uppercase{Experiments}}\label{sec:exp}

In this section, we aim at evaluating our algorithms empirically and compare them to prior work. The goal is to validate the theory in practice using synthetic and real-world data. In particular, we are trying to answer the following research questions;
\begin{enumerate}   
    \item How does our non-contextual algorithm \cref{alg:Gauss-LTR} compare to prior work on non-contextual LTR?
    We investigate this in \cref{fig:stand}.
    
    \item How does prior initialization and misspecification affect our non-contextual algorithm? \cref{fig:stand-misspec} and  further experiments of \cref{fig:stand-subset,fig:prior-standard-cut,fig:prior-standard} in \cref{app:further-experiments} answer this.
    
    \item For a linear or logistic reward model, how does \cref{alg:LinGauss-LTR,alg:Log-LTR} compare to prior work in terms of regret? We conduct the corresponding experiments for the linear and logistic model in \cref{fig:lin,fig:log}, respectively. \cref{fig:lin-subset,fig:log-cut} experiments in \cref{app:further-experiments} explore this further.
    
    \item How would our algorithms perform if the theoretical assumptions are violated? We explore this using two datasets in \cref{fig:web30k,fig:Istella}. Further experiments \cref{fig:web30k-cut,fig:web30k-all,fig:Istella-cut,fig:Istella-all} are in \cref{app:further-experiments}.
\end{enumerate}

The default setting for our synthetic experiments is $L=30$, $K=3$, $T=10000$, $S=1$, and $\lambda=1/10^{4}$. \jtodo{Hyper-parameters tuning discussion}Our preliminary analysis shows that the few hyper-parameters ($\lambda, S$) of our algorithms do not change the behavior much and thus we did not include a tuning study. This is also in line with the regret bounds which show the effect of the hyper-parameters is negligible. We compare the cumulative regret of the algorithms over 100 simulation replications. In the synthetic experiments, each replication is a new instance sampled from a fixed prior.

\begin{figure*}[tb]
\centering
    \begin{minipage}[t]{.32\textwidth}
    \includegraphics[width=\textwidth]{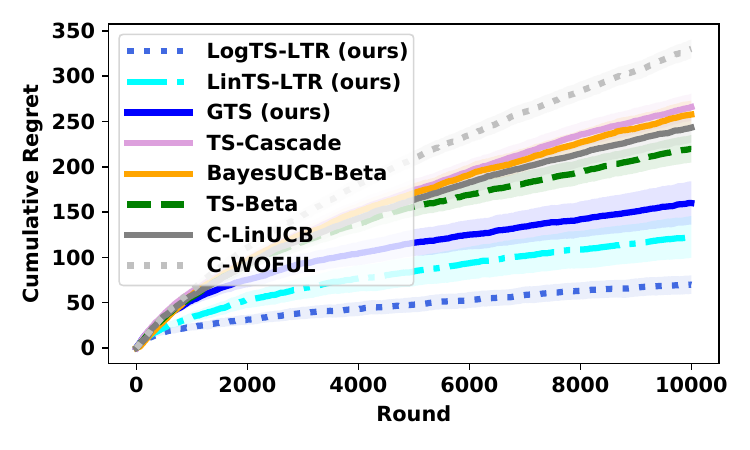}
    \caption{Logistic setting experiment.}
    \label{fig:log} %
    \end{minipage}
\hfill
    \begin{minipage}[t]{.32\textwidth}
     \includegraphics[width=\textwidth]{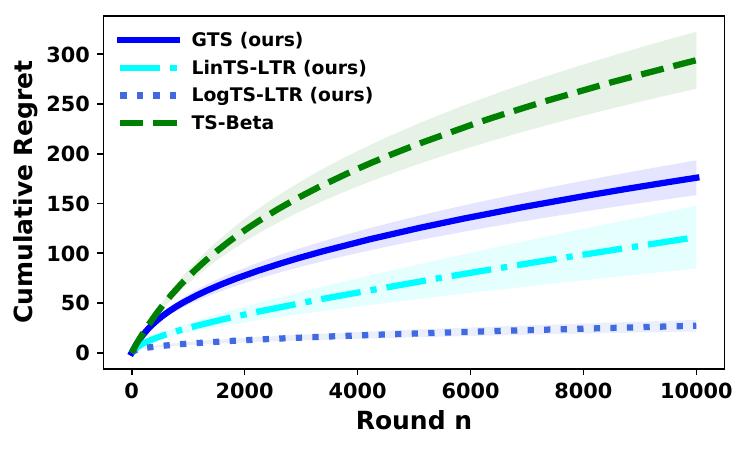}
     \caption{Web30k experiment with $\listlen=10$.}
     \label{fig:web30k} %
    \end{minipage}
\hfill
    \begin{minipage}[t]{.32\textwidth}
     \includegraphics[width=\textwidth]{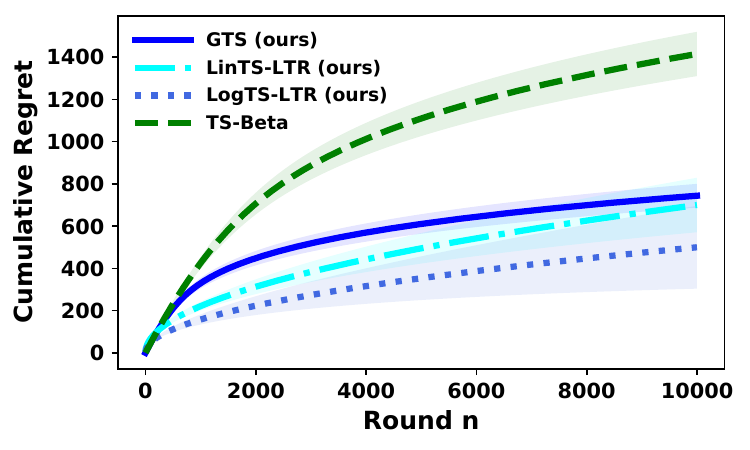}
     \caption{Istella experiment with $\listlen=10$.}
     \label{fig:Istella} %
    \end{minipage}
\end{figure*}

Our \cref{alg:Gauss-LTR} is shown as \algGTS, \cref{alg:LinGauss-LTR} as \algLinTSLTR, and \cref{alg:Log-LTR} with Sigmoid link function as \algLogTSLTR. Here, \algGPTSmean is \algGTS with a prior mean initialized to the mean of the offline data %
(instead of $\Ngaus(0,1)$ uninformative prior). \algGPTS uses both the mean and variance of the offline data for the prior initialization.

\subsection{Synthetic Experiment: Non-contextual Setting}\label{sec:exp-stand}

In the non-contextual setting, our baselines are \algTS (shown as \algTSB to emphasis beta updates versus Gaussian) and \algBUCB (shown as \algBUCBB) from \citet{Bkveton-22-prior} which are the state-of-the-art and shown to outperform the other algorithms. These algorithms use the exact prior so we call them \emph{prior-informed} algorithms. The rest of the algorithms here which do not use the prior are called \emph{prior-agnostic}. For instance, we include \algTScascade which is the closest algorithm to \algGTS in the literature from \citet{Zhong-TS_cascade-2021} and uses dependent posterior samples in a semi-TS algorithm (see \cref{sec:related-works}). 

For the non-contextual setting, in \cref{fig:stand} (and \cref{fig:stand-subset,fig:prior-standard-cut,fig:prior-standard} in \cref{app:further-experiments}), we let $(\beta_{1,i}, \beta_{2,i})$ be the Beta prior distribution parameters (used by \algTS and \algBUCB). We let $\beta_{2,i} = 10$ and sample $\beta_{1,i}$ from [10] uniformly at random,
for all items $i$. We sample $(\beta_{1,i})_{i=1}^{\nitems}$ for 20 times, and for
each one sample the attraction probability of items as $\attMu_i \sim \Betad(\beta_{1,i},\beta_{2,i})$ 20 times. Each algorithm is simulated on these 400 bandit instances. This procedure generates attraction probabilities between 0.09 and 0.5, so no item is overly attractive.

\cref{fig:stand} shows the cumulative regret for all the algorithms on a cascade model. We can observe that \algGTS competes closely with \algTSB and \algBUCBB, despite not having access to the true prior. Also, \algTScascade seems to underperform significantly. We conjecture this could be due to its dependent posterior samples which could pollute its posterior statistics (not MAP estimates). In \cref{fig:stand-subset} of \cref{app:further-experiments} we removed \algTScascade to zoom in on the difference between \algGTS's. \algGPTSmean and \algGPTS correct \algGTS at the beginning as they start with a better prior mean, the gain, however, is not significant which shows our \ac{TS} algorithms are highly adaptive. \cref{fig:prior-standard-cut,fig:prior-standard} in \cref{app:further-experiments} confirm this observation further. 
Also, \algGPTS shows slightly worse performance which seems to highlight a drawback of initializing a Gaussian \ac{TS} with a beta prior. 

In \cref{fig:stand-misspec}, the true prior is
$Beta(1, 10)$ but the prior fed to the algorithms is $\Betad(1 + c, 10 - c)$, with $c \in [0, 8]$. Therefore, when $c > 0$, the prior is misspecified. We plot the final cumulative regret of the algorithms at $\timH=1000$. This figure illustrates that our algorithm\footnote{\algGPTSmean or \algGTS both perform similarly here so we only keep \algGPTSmean} is robust against the prior misspecification, while the performance of prior-dependant algorithms degrades, as expected.

\subsection{Synthetic Experiment: Linear Setting}\label{sec:exper-lin}

In the linear setting, the state-of-the-art baseline algorithm is \algcascadeWOFUL \citep{vial2022minimax}, labeled as \algCWOFUL, which deploys a Bernstein UCB algorithm and adapts to the variance of the rewards. 
We removed \algBUCBB in some experiments as it has almost the same performance as \algTSB. 

In this section, we sample $\cparam\in[0,1]^d$ and $\cntx_{i,t}\in[0,1]^d$ from a uniform distribution and normalize them to unit norm. This aligns with the experimental setup of the linear LTR algorithm in prior work \citep{Zhong-TS_cascade-2021,vial2022minimax}. 

\cref{fig:lin} demonstrate the cumulative regret of the state-of-the-art linear algorithm and our \algLinTSLTR. It shows that our algorithm and \algLinTSCas \citep{Zhong-TS_cascade-2021} outperform the prior work. \cref{fig:lin-subset} in \cref{app:further-experiments} separates the underperforming algorithms and illustrates that \algLinTSCas and \algLinTSLTR outperform all others including \algGTS, showcasing the benefit of including the contextual information.

\subsection{Synthetic Experiment: Logistic Setting}

For the logistic setting, as far as we know there are no comparable prior works for the cascade model. We include \algcascadeLinUCB (shown as \algCLinUCB) for completeness in the experiments, although it is empirically shown to underperform \algcascadeWOFUL. In this section, we sample $\cparam$ and $\cntx_i$ in the same way as in \cref{sec:exper-lin}, but set $\attMu_i=\linkf(\cparam\T\cntx_i)$. We set $\beta_{2,i}=10$ and $\beta_{1,i}=10\attMu_i/(1-\attMu_i)$ so the mean of Beta prior of item $i$ is $\attMu_i$.\todoOM{$\attMu_i$?}

We show the cumulative regret of our algorithms in \cref{fig:log}, where our algorithm's superior performance is highlighted. A zoomed-in version of this in \cref{fig:log-cut} in \cref{app:further-experiments} shows \algLogTSLTR outperforms other algorithms (including the prior-informed ones). Moreover, even when considering only the first 1000 steps, \algLogTSLTR still outperforms the rest, showing that it is highly adaptive.

\subsection{Real-world Experiments} 

In this section, we first experiment with the Microsoft Learning to Rank
{\bf Web30K} dataset\footnote{\url{https://www.microsoft.com/en-us/research/project/mslr/}} \citep{qin2013introducing}, which contains
$\sim$18000 training and $\sim$6000 test queries, with an average of $L = 120$ documents per query. Each query and document pair has 136 features which we normalize and take as the context. We follow the setting in \citet{Bkveton-22-prior}. In particular, priors are generated from this dataset by training $M = 10$ state-of-the-art deep-learning LTR model \citep{qin2021neural}, each on a randomly selected training subset. This gives $M$ scores for each query-document pair $i$, $(s_{i,j})_{j=1}^M\in[0,1]$. Then, $\Betad(\sum_{j=1}^m s_{i,j}, \sum_{j=1}^m (1-s_{i,j}) )$ is used as the Beta prior. See \citet[Section 7.4]{Bkveton-22-prior} for more details. 

In \cref{fig:web30k} we depict the regret of the most competitive algorithms in this dataset for 50 simulation replications on each query in the test dataset. We observe that \algGTS\ has superior performance compared to \algTSB. Also, \algLogTSLTR outperforms other algorithms by using the context and its nonlinear power. It is interesting to see how \algLinTSLTR performs better than \algGTS but worse than \algLogTSLTR, as expected. \cref{fig:web30k-cut} in \cref{app:further-experiments} shows a closer look of this result. \cref{fig:web30k-all} in \cref{app:further-experiments} shows the performance of other algorithms that underperform the ones here for completeness. %

Next, we experiment with {\bf Istella}.\footnote{\url{http://blog.istella.it/istella-learning-to-rank-dataset/}} This is one of the largest publicly available learning-to-rank datasets, particularly useful for large-scale experiments on the efficiency and scalability of LETOR solutions. It is composed of 33,018 queries and 220 features representing each query-document pair. It consists of 10,454,629 examples (query-document) labeled with relevance judgments ranging from 0 (irrelevant) to 4 (perfectly relevant). The average number of per-query examples is $L=316$. This dataset is split into train and test sets according to an 80\%-20\% scheme. 

\cref{fig:Istella} shows the results of our Istella experiment. We illustrate the regret of the most competitive algorithms in this dataset for 50 simulation replications on each query in the test dataset. We can see that \algLogTSLTR still outperforms \algLinTSLTR, \algGTS, and \algTSB, but by a smaller margin. We also observe the higher variance in this experiment compared to Web30K. Our conjecture is that this is due to the high sparsity of the features in this dataset. \cref{fig:Istella-cut} in \cref{app:further-experiments} shows a closer look of this result. \cref{fig:Istella-all} in \cref{app:further-experiments} shows that \algcascadeWOFUL underperforms the other algorithms here.

\jtodo{add more detail to real world XPs, Reviewer 1}
In summary, the main finding in this section is that \algLogTSLTR outperforms all other
4 algorithms (cf. \cref{fig:web30k,fig:Istella}). This suggests that these real-world settings could have common non-linear characteristics with the logistic setting (\cref{fig:log}), even though we do not control their true reward.

\section{CONCLUSION}

We proposed highly adaptive Gaussian \ac{TS} algorithms for online learning to rank that handle prior misspecification. Our algorithms exploit contextual information. We extend prior work on cascading bandit to new forms of relevance feedback other than clicks. In particular, our linear model handles scalar feedback. Our Bayes regret upper bounds are the first of their own sort, and in the non-contextual setting our bound is near-optimal. Finally, our extensive experiments on synthetic and real-world datasets demonstrate the efficacy of our algorithms.

Prior-dependent regret bounds are of interest for future work. We can achieve this type of bounds using the information-theoretic analysis in \citet{liu2022gaussianImagination,Lu-InfoBds-19}. Extending our work to other click models such as \emph{dependant click model} (DCM) \citep{guo2009efficient} in the contextual framework seems promising \citep{liu2018contextual-GL-CDCM,santara22a,ermis2020learning-PBM}. \citet{vial2022minimax} developed a GLM bandits algorithm based on UCB for DCM and it would be interesting to see how our \ac{TS} algorithm would compare to it in DCM. Providing a lower bound for the contextual setting is of interest as well.

Meta-learning the prior in a hierarchical structure \citep{hong2022deep} over the queries (cluster the queries) could potentially alleviate the cold-start problem too. Approximate posterior sampling algorithms \citep{dumitrascu2018pg} generalize our \ac{TS} algorithm to a vast spectrum of feedback distribution beyond Gaussian and Bernoulli which we deal with here. Improved alternatives for posterior sampling such as Feel good \ac{TS} \citep{zhang2022feel} and Maillard Sampling algorithm \citep{bian2022maillard} have the potential to further improve the performance of our algorithms.

\bibliography{refs}

\clearpage
\appendix
\onecolumn

\section{\uppercase{Gaussian TS Proofs}}\label{app:GTS-proofs}

\GaussTS*

\begin{proof}[Proof of \cref{thm:GaussTS}]
    Let 
    \begin{align}
        \Upr_{k,t} = \hat{\attMu}_{k,t} + \sqrt{\frac{\log_+(\frac{\timH}{\nitems \nobs_{k,t} })}{\nobs_{k,t} }}\;,
    \end{align}
    where $\log_+(x)=\log(x)\indc(x\geq 1)$. Now, notice that
    \begin{align}
    \begin{split}
        \E_t[\Upr_{\act^*(k),t}] &=
        \E_t[\Upr_{\amaxk(\attMu),t}] 
        \\&=
        \E_t[\Upr_{\amaxk(\hat{\attMu}),t}]
        \\&=
        \E_t[\Upr_{\act_t(k),t}]\;,
    \end{split}
        \label{eq:id-TS}
    \end{align}
    where $\amaxk(\cdot)$ is the $k$'th largest element operator. In the first and last equations, we used the definition of $\Upr$ and the second equation is by the construction of the algorithm, as $\attMu$ and $\tilde{\attMu}_t$ are identically distributed conditionally on $\hist_t$. 
    
    Now, by \citet[Lemma 7]{Bkveton-22-prior} we know that
\begin{align}
    \E_t[\reg_t]&=\E_t\bigg[\prod_{k\in\act^*}(1-\attMu_{k})-\prod_{k\in\act_t}(1-\attMu_{k})\bigg]\nonumber
    \\
    &=\E_t\bigg[\sum_{k=1}^{\listlen}\bigg(\prod_{j=1}^{k-1}(1-\attMu_{\act^*(j)})\bigg)(\attMu_{\act^*(k)}-\attMu_{\act_t(k)})\bigg(\prod_{j=k+1}^{{\listlen}}(1-\attMu_{\act_t(j)})\bigg)\bigg]\nonumber
    \\
    &\le 
    \E_t\bigg[\sum_{k=1}^{\listlen} \attMu_{\act^*(k)}-\attMu_{\act_t(k)}\bigg]
    \label{eq:R2theta}\;.
\end{align}
    Thus, for CM we have
    \begin{align*}
        \breg(\timH) &\eqdef \E[\sum_{t=1}^\timH \mrew_{\act^*}-\mrew_{\act_t}]
        \\
        &\le \E\bigg[\sum_{t=1}^\timH \sum_{k=1}^{\listlen} \attMu_{\act^*(k)}-\attMu_{\act_t(k)}\bigg]
        \tag{by \cref{eq:R2theta}}
        \\
        &=\E\bigg[\sum_{t=1}^\timH \E_t\Big[\sum_{k=1}^{\listlen} \attMu_{\act^*(k)}-\Upr_{\act^*(k),t}+\Upr_{\act_t(k),t}-\attMu_{\act_t(k)}\Big]\bigg]\tag{by \cref{eq:id-TS} and the tower rule of expectation}
        \\
        &=\E\bigg[\sum_{t=1}^\timH \sum_{k=1}^{\listlen} \E_t\Big[ \attMu_{\act^*(k)}-\Upr_{\act^*(k),t}\Big]+\sum_{k=1}^{\listlen}\E_t\Big[\Upr_{\act_t(k),t}-\attMu_{\act_t(k)}\Big]\bigg]
    \end{align*}
    Let's assume $\nitems \le \timH$, and let $\delta_0=2\sqrt{\frac{\nitems}{\timH}}$. We know by the integration by parts formula
    \begin{align}
        \E_t\Big[ \attMu_{\act^*(k)}-\Upr_{\act^*(k),t}\Big] \le \delta_0 + \int_{\delta_0}^1 \P(\attMu_i-\Upr_{i,t}\geq u)\dr u\nonumber
    \end{align}
    We can use the following inequality from \citet{JMLR:v11:audibert10a} which holds for any $i\in[\nitems]$
    \begin{align}
        \P(\attMu_i-\Upr_{i,t}\geq u)\le \frac{16 \nitems}{\timH u^2}\log(\sqrt{\frac{\timH}{\nitems}}u)+\frac{1}{\timH u^2/\nitems - 1}\nonumber
    \end{align}
    then using the same reasoning as in \citet[Theorem 1]{bubeck2013prior}, we can show by integration that
    \begin{align*}
        \int_{\delta_0}^1 \frac{16 \nitems}{\timH u^2}\log(\sqrt{\frac{\timH}{\nitems}}u) \dr u
        &= \left[- \frac{16 \nitems}{\timH u} \log \left(e \sqrt{\frac{\timH}{\nitems}} u\right) \right]_{\delta_0}^1
        \\&\le \frac{16 \nitems}{\timH \delta_0} \log \left(e \sqrt{\frac{\timH}{\nitems}} \delta_0\right) 
        \\&= 8(1+ \log 2) \sqrt{\frac{\nitems}{\timH}} 
    \end{align*}
    and
    \begin{align*}
       \int_{\delta_0}^1 \frac{1}{\timH u^2 / \nitems - 1}\dr u &= \left[- \frac{1}{2} \sqrt{\frac{\nitems}{\timH}} \log \left( \frac{\sqrt{\frac{\timH}{\nitems}} u + 1}{\sqrt{\frac{\timH}{\nitems}} u - 1}\right) \right]_{\delta_0}^1
       \\&\le \frac{1}{2} \sqrt{\frac{\nitems}{\timH}} \log \left( \frac{\sqrt{\frac{\timH}{\nitems}} \delta_0 + 1}{\sqrt{\frac{\timH}{\nitems}} \delta_0 - 1}\right)
       \\&= \frac{\log 3}{2} \sqrt{\frac{\nitems}{\timH}} 
    \end{align*}
    Thus, putting these together we get
    \begin{align}
        \E_t\Big[ \attMu_{\act^*(k)}-\Upr_{\act^*(k),t}\Big]\le 13\sqrt{\frac{\nitems}{\timH}},\qquad \forall k\in[{\listlen}]\;.\nonumber
    \end{align}
    Then taking the Bayes regret expectation over $\attMu$, and \cref{lem:CS}, we get
    \begin{align}
        \E\bigg[\sum_{t=1}^\timH \sum_{k=1}^{\listlen} \E_t\Big[ \attMu_{\act^*(k)}-\Upr_{\act^*(k),t}\Big]\bigg]\le 13 \listlen \sqrt{ \nitems \timH}\;.\label{eq:actStar-dev}
    \end{align}
    
    Similarly, by integration by parts, we again know for any $k\in[{\listlen}]$
    \begin{align}
        \sum_{t=1}^\timH \E_t\Big[\Upr_{\act_t(k),t}- \attMu_{\act_t(k)}\Big] 
        \le 
        \delta_0 \timH + \int_{\delta_0}^\infty \sum_{t=1}^\timH   \P(\Upr_{\act_t(k),t}-\attMu_{\act_t(k)}\geq u)\dr u \nonumber\;.
    \end{align}
    By definition of $\Upr_{\act_t(k),t}$ and a union bound style argument we can write for any $k\in[{\listlen}]$
    \begin{align}
        \sum_{t=1}^\timH \P(\Upr_{\act_t(k), t}-\attMu_{\act_t(k)}\geq u)\le \sum_{t=1}^\timH \sum_{i=1}^\nitems \P \bigg( \hat{\attMu}_{i,t} + \sqrt{\frac{\log_+(\frac{\timH}{\nitems \nobs_{i,t} })}{\nobs_{i,t} }} -\attMu_i \geq u \bigg)
        \nonumber\;.
    \end{align}
    Now we fix $i$, let $s(u)=\big\lceil \frac{3\log(\frac{\timH u^2}{\nitems})}{u^2}\big\rceil$ for $u\geq \delta_0$, and $c=1-\tfrac{1}{\sqrt{3}}$, then we can write
    \begin{align}
        \sum_{t=1}^\timH \P\bigg(\hat{\attMu}_{i,t} + \sqrt{\frac{\log_+(\frac{\timH}{\nitems \nobs_{i,t} })}{\nobs_{i,t} }} -\attMu_i \geq u \bigg) 
           &
        \le s(u) + \sum_{t = s(u)}^\timH \P(\hat{\attMu}_{i,t} - \attMu_i \geq c u) \nonumber%
        \\ &
        \le s(u) + \sum_{t = s(u) }^\timH \exp(-2 t c^2 u^2)\indc(u\le 1/c) \tag{Hoeffding's inequality}
        \\ &
        \le s(u) + \frac{\exp(-12 c^2 \log 2)}{1-\exp(-2 c^2 u^2)}\indc(u\le 1/c) \nonumber\;.
    \end{align}
    Now note that 
    \begin{align}
        \int_{\delta_0}^{+\infty} \frac{3 \log \left(\frac{\timH u^2}{\nitems}\right)}{u^2} \dr u &\le 3 (1+\log(2)) \sqrt{\frac{\timH}{\nitems}} \nonumber
        \\&\le 5.1 \sqrt{\frac{\timH}{\nitems}}\nonumber
    \end{align}
    and by the fact that $1-\exp(-u)\geq u - u^2/2$ for $u\geq 0$ we can write
    \begin{align*}
        \int_{\delta_0}^{1/c} \frac{1}{1 - \exp(- 2 c^2 u^2)} \dr u & =  \int_{\delta_0}^{1/ (2 c)} \frac{1}{1 - \exp(- 2 c^2 u^2)} \dr u + \int_{1/(2 c)}^{1/c} \frac{1}{1 - \exp(- 2 c^2 u^2)} \dr u
        \\
        & \le   \int_{\delta_0}^{1/ (2 c)} \frac{1}{2 c^2 u^2 - 2 c^4 u^4} \dr u + \frac{1}{2 c(1 - \exp(- 1/2))}  
        \\
        & \le \int_{\delta_0}^{1/ (2 c)} \frac{2}{3 c^2 u^2} \dr u + \frac{1}{2 c(1 - \exp(- 1/2))} 
        \\
        & = \frac{2}{3 c^2 \delta_0} - \frac{4}{3 c} + \frac{1}{2 c(1 - \exp(- 1/2))} 
        \\
        & \le 1.9 \sqrt{\frac{\timH}{\nitems}}.
    \end{align*}
    which altogether means for any $k\in[{\listlen}]$ we get
    \begin{align}
        \sum_{t=1}^\timH \E_t\Big[\Upr_{\act_t(k),t}- \attMu_{\act_t(k)} \Big]\le 9  \sqrt{\nitems \timH}\nonumber\;.
    \end{align}
    Now again by \cref{lem:CS} we get
    \begin{align}
        \E\bigg[\sum_{k=1}^{\listlen} \sum_{t=1}^\timH \E_t\Big[\Upr_{\act_t(k),t}- \attMu_{\act_t(k)} \Big]\bigg]\le 9 \listlen  \sqrt{ \nitems \timH}\;.\label{eq:Atk-dev}
    \end{align}
    Finally, \cref{eq:actStar-dev} and \cref{eq:Atk-dev} together show $\breg(\timH) \le \bndOne$ for \cref{alg:Gauss-LTR}.
\end{proof}

\subsection{Gaussian vs. Beta \ac{TS}}\label{app:Gaus-vs-Bern}

In this section we compare the Gaussian posterior to the Bernoulli posterior in a \ac{TS} algorithm. The Gaussian version updates its posterior as given in \cref{eq:posterior-GTS}, while Beta does it as follows:

\begin{align}
    & \alpha_{i,t} = \alpha_i+\sum_{\ell=1}^t \att_{i,\ell}\obsind_{i,t}, \ \ \ 
    \beta_{i,t} = \beta_i+\nobs_{i,t}-\sum_{\ell=1}^t \att_{i,\ell}\obsind_{i,t}
    \;,\nonumber
\end{align}
and the mean and variance of Beta posterior are (lose notation here)
\begin{align}
    \hat{\attMu}_{i,t} &= \frac{\alpha_{i,t}}{\alpha_{i,t}+\beta_{i,t}}
    \\
    \hat{\sigma}_{i,t}^2 &= \frac{\alpha_{i,t} \beta_{i,t} }{(\alpha_{i,t}+ \beta_{i,t})^2(\alpha_{i,t}+\beta_{i,t}+1)}
    \;.
\end{align}
It is not hard to see both posterior parameters of item $i$ converge to mean $O(\frac{ \sum_{\ell=1}^t \att_{i,\ell}\obsind_{i,t} }{ \nobs_{i,t} })$ and variance $O(1/\nobs_{i,t})$ for large $t$, almost with the same rate. However, the difference is in their certainty and concentration around the true mean. \cref{fig:BvsGPost} illustrates the posterior distribution for different true means ($\attMu\in\{0.1, 0.5, 0.9\}$) after $\obsind_{i,t}=100$ (so the variance is $1/100$). As we can see for the extreme cases of attractiveness ($\attMu$ close to 0 or 1), the Beta distribution does not concentrate and its uncertainty is very high (indeed, the entropy = $\infty$ for Beta posterior). The Gaussian posterior is however suitable for different values of $\attMu$. The extreme values of $\attMu$ are more important in ranking problems as the best items are highly attractive (large $\attMu$) and choosing a very unattractive item (small $\attMu$) could result in linear regret. Also, in practice usually $\attMu$ is small, so the posterior must be very accurate around zero. We leave the theoretical establishment of this observation and its benefits in overcoming the prior misspecification to the future work. 

\begin{figure*}[tb]
\centering
    \begin{minipage}[t]{.32\textwidth}
    \includegraphics[width=\textwidth]{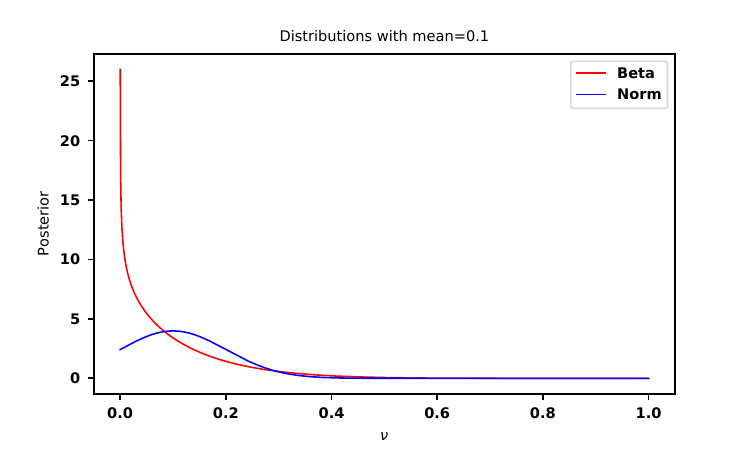}
    \end{minipage}
\hfill
    \begin{minipage}[t]{.32\textwidth}
    \includegraphics[width=\textwidth]{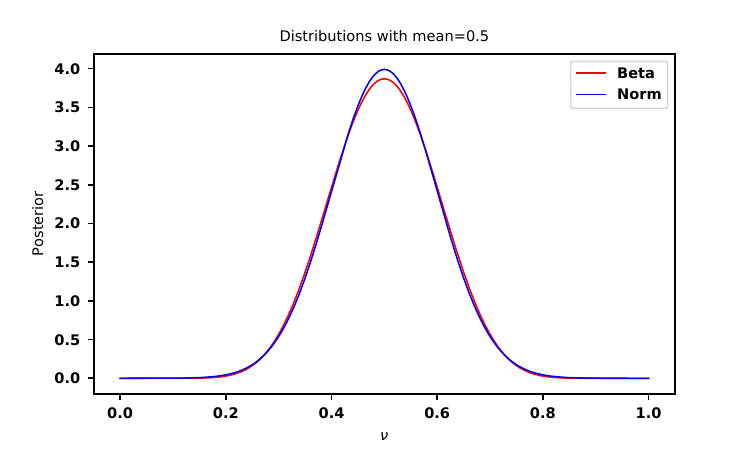}
    \end{minipage}
\hfill
    \begin{minipage}[t]{.32\textwidth}
    \includegraphics[width=\textwidth]{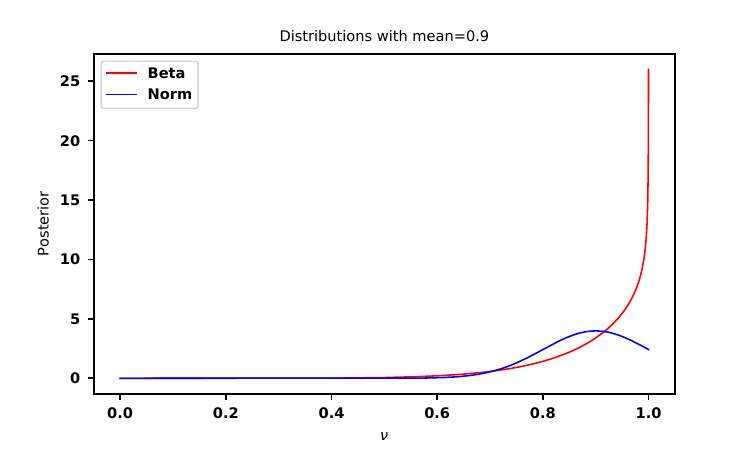}
    \end{minipage}
\caption{Concentration of Beta vs. Gaussian posterior after $100$ samples.}
\label{fig:BvsGPost}
\end{figure*}

\section{\uppercase{Linear \ac{TS} Proofs}}\label{app:LinTS-proofs}

\LinLTR*
\begin{proof}[Proof of \cref{thm:LinTS}]%
    First we write out the conditional step regret as follows
    \begin{align}
        \E_t[\reg_t]&= 
        \E_t\bigg[\sum_{k=1}^{\listlen} \cntx_{\act^*(k),t}\T\cparam - \cntx_{\act_t(k),t}\T\cparam\bigg]\label{eq:R2theta-cntx}
    \end{align}
    Notice that by the construction of the algorithm we know 
    \begin{align}
    \begin{split}
        \E_t[\cntx_{\act^*(k),t}\T\cparam] =
        \E_t [\cntx_{\act_t(k),t}\T\tilde{\cparam}_t]\;,
    \end{split}
        \label{eq:iden-dist-TS-cntx}
    \end{align}
    for all $k,t$. This is because given $\hist_t$, $\tilde{\cparam}_t$ is identically distributed as $\cparam$ and $\act_t$ is a deterministic function of $\tilde{\cparam}_t$ (See \citet[Section 4]{abeille2017linear} and \citet[Proof of Theorem 1, Step 1]{bubeck2013prior}).
    
    Now, for the cascade model we have
    \begin{align}
        \breg(\timH) &\eqdef \E[\sum_{t=1}^\timH \mrew_{\act^*_{t}}-\mrew_{\act_t}]\nonumber
        \\
        &\le \E\bigg[\sum_{t=1}^\timH \sum_{k=1}^{\listlen} \cntx_{\act^*(k),t}\T\cparam - \cntx_{\act_t(k),t}\T\cparam
        \bigg]
        \tag{by \cref{eq:R2theta-cntx}}
        \\
        &=\E\bigg[\sum_{t=1}^\timH \E_t\Big[\sum_{k=1}^{\listlen} \cntx_{\act_t(k),t}\T\tilde{\cparam}_t - \cntx_{\act_t(k),t}\T\cparam
        \Big]\bigg]\tag{by \cref{eq:iden-dist-TS-cntx} and tower rule}
        \\
        &=\E\bigg[\sum_{t=1}^\timH \sum_{k=1}^{\listlen} 
        \E_t\Big[ \cntx_{\act_t(k),t}\T(\tilde{\cparam}_t - \hat{\cparam}_t)
        \Big]+
        \E_t\Big[
        \cntx_{\act_t(k),t}\T(\hat{\cparam}_t - \cparam)
        \Big]\bigg]\label{eq:decompLinTS}
    \end{align}

    Let's define
    \begin{align}
        \EvHat{t} \eqdef& \{\forall \cntx\in\Cntx_t, |\cntx\T(\hat{\cparam}_t-\cparam)|\le \|\cntx\|_{\infMat_t^{-1}}\radi_t(\delta_t)\}\;,\nonumber
        \\
        \EvTilde{t} \eqdef& \{\forall \cntx\in\Cntx_t, |\cntx\T(\hat{\cparam}_t-\tilde{\cparam})|\le \|\cntx\|_{\infMat_t^{-1}}\radiTwo_t(\delta_t)\}\;.\nonumber
    \end{align}
    where $\radiTwo_t(\delta)=\radi_t(\delta)\sqrt{c d \log(c' d/\delta)}$ with $c,c'$ constants such that $\P\left(\|\att_{i,t}-\cntx_{i,t}\T\cparam\|\le \sqrt{c d\log(c' d/\delta)}\right)\geq 1-\delta$. Note that based on \cref{asm:noise} and Hoeffding's inequality, we can set $c=c'=2$ (see \citet[Appendix A, Example 2]{abeille2017linear}).
    
    \citet[Theorem 2]{abbasi2011improved} shows that $\P(\EvHat{t})=1-\delta_t$, so a union bound along with \cref{lem:DT} shows $\P(\cap_{t=1}^\timH \EvHat{t})\geq 1-\delta(\frac{\log\timH}{2}+1)$. 
    
    \cite[Lemma 1, the proof in the Appendix D]{abeille2017linear} states that $\P(\EvTilde{t})=1-\delta_t$. Again, a union bound along with \cref{lem:DT} shows $\P(\cap_{t=1}^\timH \EvTilde{t})\geq 1-\delta(\frac{\log\timH}{2}+1)$.
    
    Thus, using the high probability bounds above, under $\cap_{t=1}^\timH \EvHat{t}$ and $\cap_{t=1}^\timH \EvTilde{t}$ we can bound \cref{eq:decompLinTS} with
    \begin{align}
        &
        \E\bigg[\sum_{t=1}^\timH \sum_{k=1}^{\listlen}
        \indc(\EvTilde{t})
        \E_t\Big[ \cntx_{\act_t(k),t}\T(\tilde{\cparam}_t - \hat{\cparam}_t)
        \Big]+
        \indc(\EvHat{t})
        \E_t\Big[
        \cntx_{\act_t(k),t}\T(\hat{\cparam}_t - \cparam)
        \Big]\bigg]
        \nonumber\\
        &\le \E\bigg[\sum_{t=1}^\timH \sum_{k=1}^{\listlen} 
        \radiTwo_t(\delta_t)\|\cntx_{\act_t(k),t}\|_{\infMat_t^{-1}}+
        \radi_t(\delta_t)\|\cntx_{\act_t(k),t}\|_{\infMat_t^{-1}}\bigg]
        \nonumber\\
        &\le \E\bigg[\listlen\sum_{t=1}^\timH 
        2\radiTwo_t(\delta_t) d\log (1+t/\lambda) +
        2\radi_t(\delta_t) d\log (1+t/\lambda)\bigg]  \tag{by \cref{prop:self_normalized_determinant}}
        \nonumber\\
        &\le 4 d \listlen ( \radiTwo_\timH(\delta_1) + \radi_\timH(\delta_1) ) \sum_{t=1}^\timH 
        \log (1+t/\lambda) \label{eq:monoton}
        \\
        &\le \bndUnderEv \tag{By \cref{lem:CS}}\;, \nonumber
    \end{align}
    where \cref{eq:monoton} is by $\radi_t(\delta)$ and $\radiTwo_t(\delta)$ being decreasing in $\delta$ and increasing in $t$, and $\delta_1>\delta_t, \forall t> 1$.
    Bounding the regret under the complement of the event that the bounds do not hold with probability $\delta(\log T+2)$, we get 
    \begin{align*}
         \breg(\timH) &\le \big(1-\delta(\log \timH +2)\big) \bndUnderEv + \delta(\log\timH+2)\timH\;.
    \end{align*}
    Now, take $\delta=\LinTSdelta$, which simplifies the previous inequality to
    \begin{align*}
        \breg(\timH)
        &\le (1-1/\timH)\bndUnderEv[\tfrac{1}{2\timH(\log \timH + 2)}] + 1 
        \\
        &=O\bigg(d \listlen \sqrt{ \log(d\timH(\log \timH))}\sqrt{\log\big( (1+\timH/\lambda)^{d/2}\timH(\log \timH)\big) }\sqrt{d \timH \log(1+\timH/\lambda)} \bigg)
        \\
        &=
        \bndLinTSO\;.
    \end{align*}
\end{proof}

\begin{lemma}[Doubling Trick]\label{lem:DT}
    \begin{align}
        \sum_{t=1}^\timH 1/2^{\max(1,\lceil\log t\rceil)}\le \frac{\log \timH}{2}+1
        \nonumber
    \end{align}
\end{lemma}
\begin{proof}
 \begin{align*}
     \sum_{t=1}^\timH 1/2^{\max(1,\lceil\log t\rceil)}
     &= 2/2 + 2/4 + 4/8 + \cdots + 2^{\lfloor\log \timH\rfloor-1}/2^{\lfloor\log \timH\rfloor} + (\timH - 2^{\lfloor\log \timH\rfloor})/2^{\lceil \log \timH\rceil}
     \\
     &= 1 + (\lfloor \log \timH\rfloor-1)/2 + (\timH - 2^{\lfloor\log \timH\rfloor})/2^{\lceil \log \timH\rceil}
     \\
     &\leq 1+ \lfloor \log \timH\rfloor/2 -1/2 + \timH/2^{\lceil \log \timH\rceil}-1/2
     \\
     &\le \lfloor \log \timH\rfloor/2 + \timH/2^{\log \timH}
     \\
     &\le \frac{\log \timH}{2}+1
 \end{align*}
\end{proof}

\begin{proposition}\label{prop:self_normalized_determinant}
Let $\lambda \geq 1$, for any arbitrary sequence $(\cntx_1, \cntx_2, \ldots, \cntx_{t})\in(\Cntx_1\otimes\cdots\otimes\Cntx_t)$
\begin{equation}\label{eq:natural_explored_direction_ls}
\sum_{s=1}^{t} \|\cntx_s\|_{\infMat_s^{-1}}^2 \le 2 \log \frac{\det(\infMat_{t+1})}{\det(\lambda I)} \le 2 d \log \Big( 1 + \frac{t}{\lambda} \Big).
\end{equation}
\end{proposition}

\section{\uppercase{Logistic \ac{TS} proofs}}\label{app:log-proofs}

In this section, we use the results in \citet{filippi2010parametric,abeille2017linear} to extend the results of the linear case to the GLM case. First, we restate the following result from Appendix F of \citet{abeille2017linear}. Let $\ca{F}_t=(\ca{F}_1,\sigma(\cntx_1,\att_t,\cdots,\cntx_t,\att_t))$ be the sigma algebra generated by the prior knowledge $\ca{F}_1$ and the history up to round $t$.

\begin{proposition}[Proposition 11, \citet{abeille2017linear}]\label{prop:GLM-Bounds}
    For any $\delta\in(0,1)$ and $t\geq 1$, under \cref{asm:cntx,asm:param,asm:noise-GLM,asm:linkFunction}, for any $\ca{F}_t$-adapted sequence of contexts $(\cntx_1,\cdots,\cntx_t)$, the estimation $\hat{\cparam}_t$ from \cref{eq:Log-Param-estimate} is such that
    \begin{align}
        \|\hat{\cparam}_t-\cparam\|_{\infMat_t}\leq \frac{\radi_t(\delta)}{\minLink}\nonumber
    \end{align}
    and 
    \begin{align}
        \forall \cntx\in\R^d, \quad \|\linkf(\cntx\T\hat{\cparam})-\linkf(\cntx\T\cparam)\| \leq \frac{\lipchLink\radi_t(\delta)}{\minLink}\|\cntx\|_{\infMat^{-1}_t}\nonumber
    \end{align}
    with probability $1-\delta$.
\end{proposition}

\LogLTRTwo*
\begin{proof}[Proof of \cref{thm:Log-LTR2}]
     Starting with \cref{eq:R2theta}, we know
     \begin{align}
         \E_t[\reg_t]&\le 
        \E_t\bigg[\sum_{k=1}^{\listlen} \attMu_{\act^*_t(k)}-\attMu_{\act_t(k)}\bigg]\nonumber
        \\&
        =\E_t\bigg[\sum_{k=1}^{\listlen} \linkf(\cntx_{\act^*_t(k),t}\T\cparam) - \linkf(\cntx_{\act_t(k),t}\T\cparam)\bigg]\nonumber
        \\&
        =\E_t\bigg[\sum_{k=1}^{\listlen} \linkf(\cntx_{\act^*_t(k),t}\T\cparam) - \linkf(\cntx_{\act_t(k),t}\T\tilde{\cparam}_t)
        -\sum_{k=1}^{\listlen} \linkf(\cntx_{\act_t(k),t}\T\tilde{\cparam}_t) - \linkf(\cntx_{\act_t(k),t}\T\cparam)
        \bigg]\nonumber
        \\&
        \le \E_t\bigg[\sum_{k=1}^{\listlen} \linkf(\cntx_{\act^*_t(k),t}\T\cparam) - \linkf(\cntx_{\act_t(k),t}\T\tilde{\cparam}_t)
        -\sum_{k=1}^{\listlen} \lipchLink\|\cntx_{\act_t(k),t}\|_{\infMat_t^{-1}}\|\tilde{\cparam}_t-\cparam\|_{\infMat^{-1}_t}
        \bigg] \label{eq:GLM-decomp}
        \;.
     \end{align}
     where \cref{eq:GLM-decomp} is by definition of $\lipchLink$ and Cauchy-Schwarz inequality. The second term is bounded the same way as in Theorem 1 of \citet{abeille2017linear} using \cref{prop:GLM-Bounds} for each $k\in[\listlen]$
     \begin{align}
         \sum_{k=1}^{\listlen} \lipchLink\|\cntx_{\act_t(k),t}\|_{\infMat_t^{-1}}\|\tilde{\cparam}_t-\cparam\|_{\infMat^{-1}_t}\leq K \frac{\lipchLink}{\minLink}\big(\radi_t(\delta')+\radiTwo_t(\delta')d\big)\sqrt{2\timH d \log(1+\frac{\timH}{\lambda})} 
         \label{eq:Reg-GLM-2}
     \end{align}
     For the first term in \cref{eq:GLM-decomp} we can use the properties of \ac{TS} and have
     \begin{align*}
         \E_t\bigg[ \sum_{k=1}^{\listlen} \linkf(\cntx_{\act^*_t(k),t}\T\cparam) - \linkf(\cntx_{\act_t(k),t}\T\tilde{\cparam}_t) \bigg] &= \E_t\bigg[ \sum_{k=1}^{\listlen} \linkf(\cntx_{\act_t(k),t}\T\cparam) - \linkf(\cntx_{\act_t(k),t}\T\tilde{\cparam}_t) \bigg] \ \ \forall k\in[\listlen]
     \end{align*}
     Now, note that if $\sup_{\cntx\in\Cntx}\cntx\cparam - \sup_{\cntx\in\Cntx}\cntx\tilde{\cparam}_t\ge 0$, then by definition of $\lipchLink$ and removing the absolute value we get
     \begin{align*}
         \linkf(\cntx_{\act_t(k),t}\T\cparam) - \linkf(\cntx_{\act_t(k),t}\T\tilde{\cparam}_t)\leq \lipchLink \big(\sup_{\cntx\in\Cntx}\cntx\cparam - \sup_{\cntx\in\Cntx}\cntx\tilde{\cparam}_t \big)\ \ \forall k\in[\listlen]
     \end{align*}
     and otherwise, by definition of $\minLink$ we (always) have
     \begin{align*}
         \linkf(\cntx_{\act_t(k),t}\T\cparam) - \linkf(\cntx_{\act_t(k),t}\T\tilde{\cparam}_t)\leq \minLink \big(\sup_{\cntx\in\Cntx}\cntx\cparam - \sup_{\cntx\in\Cntx}\cntx\tilde{\cparam}_t \big) \ \ \forall k\in[\listlen]
     \end{align*}
     By the bound on $R^{\text{TS}}$ in Theorem 1 of \citet{abeille2017linear} we know
     \begin{align*}
         \sup_{\cntx\in\Cntx}\cntx\cparam - \sup_{\cntx\in\Cntx}\cntx\tilde{\cparam}_t
         \le \frac{2 \radiTwo_t(\delta')}{0.1 \minLink}\E_t\big[\|\cntx_{\act_t(k),t}\|_{\infMat_t^{-1}}\big] \ \ \forall k\in[\listlen]
     \end{align*}
     (Note that $p$ from \citet{abeille2017linear} is replaced by 0.1 accordingly)
     where $\delta'=\frac{\delta}{4\timH}$. Next we can use Proposition 2 of \citet{abeille2017linear} and get
     \begin{align*}
        \E_t\bigg[ \sum_{t=1}^\timH \sum_{k=1}^\listlen \|\cntx_{\act_t(k),t}\|_{\infMat_t^{-1}} \bigg] \le \listlen \sqrt{\frac{8\timH}{\lambda}\log\frac{4}{\delta}}\;,
     \end{align*}
     which means 
     \begin{align}
         \sum_{t=1}^\timH \E_t\bigg[ \sum_{k=1}^{\listlen} \linkf(\cntx_{\act_t(k),t}\T\cparam) - \linkf(\cntx_{\act_t(k),t}\T\tilde{\cparam}_t) \bigg] \le \listlen \frac{2 \radiTwo_t(\delta')}{0.1 \minLink} \sqrt{\frac{8\timH}{\lambda}\log\frac{4}{\delta}} \label{eq:reg-GLM-1}
     \end{align}
     Putting it all together (\cref{eq:Reg-GLM-2} and \cref{eq:reg-GLM-1}) we get
     \begin{align*}
         \breg(\timH)=\sum_{t=1}^\timH \E_t[\reg_t] \leq 
         \bndLogTSKappa
     \end{align*}
     holds with probability at least $1-\delta$. Now, if we take $\delta=\frac{1}{\timH}$ and $\gamma=1$, the regret upper bound is 
     \begin{align*}
         \bndLogTSKappaOrderwise
     \end{align*}
     
\end{proof}

\subsection{Removing the Dependence on $\minLink$ and $\lipchLink$}\label{app:log-proofs-noKappa}
In this section, we prove the result in \cref{thm:LogLTR1} for the Sigmoid function as the link function.
\LogLTROne*
\begin{proof}[Proof of \cref{thm:LogLTR1}]
    We define the \emph{partial lifted information gain}  \citep{neu2022lifting} as
    \begin{equation}
        \infoRatio_{t,k} \eqdef \frac{(\E_t[\mrew_{\act^*(k)}(\cparam,\cntx_t)-\mrew_{\act_t(k)}(\cparam,\cntx_t)])^2}{\info_t(\cparam; \att_{t,\act_t(k)})}
        \nonumber
    \end{equation}
    where $\mrew_{\act(k)}(\cparam,\cntx)$ is the probability of click (mean reward) for item $\act(k)$ under context $\cntx$. Also, 
    \[\info_t(\cparam; \att_{t,\act(k)}) \eqdef \E_t\big[\KL{\P(\att_{t,\act(k)}|\cparam,\cntx_t,\act(k),\hist_t)}{\P(\att_{t,\act(k)}|\cntx_t,\act(k),\hist_t)}\big]\]
    is the mutual information between $\cparam$ and $\att_{t,\act(k)}$ conditioned on the history, context $\cntx_t$, and action $\act_t$. Here $\KL{p}{q}$ is the KL divergence between $p$ and $q$ distributions. 
    
    Since $\linkf(\cdot)\in[0,1]$, we can write the step regret conditioned on the history using \citet[Lemma 7]{Bkveton-22-prior}
    as follows
    \begin{align}
        \E_t[\reg_t] 
        &\le \E_t\bigg[\sum_{k=1}^{\listlen} \linkf(\cntx_{\act^*(k),t}\T\cparam) - \linkf(\cntx_{\act_t(k),t}\T\cparam)\bigg]
        \nonumber
        \\
        &= \E_t\bigg[\sum_{k=1}^\listlen \sqrt{\infoRatio_{t,k}   \info_t(\cparam; \att_{t,\act(k)})}\bigg]
        \label{eq:reg-infoGain-bnd}\;,
    \end{align}
    where in the last equality we used the definition of $\infoRatio$ and $\mrew$.
    Now by tower rule, we have
    \begin{align*}
        \breg(\timH)
        &=
        \E\big[\sum_{t=1}^\timH \E_t[\reg_t]\big]
        \\
        &\le \sqrt{
        \E\bigg[\sum_{t=1}^\timH\sum_{k=1}^\listlen \infoRatio_{t,k}\bigg] 
        ~\E\bigg[\sum_{t=1}^\timH \sum_{k=1}^\listlen \info_t(\cparam;\att_{t,\act_t(k)})\bigg]
        }\;.
        \nonumber
    \end{align*}
    where the inequality is by \cref{lem:CS} and \cref{eq:reg-infoGain-bnd}. 
    Now by \cref{lem:Info-Ratio} we get
    \begin{align*}
        \E\bigg[\sum_{t=1}^\timH\sum_{k=1}^\listlen \infoRatio_{t,k}\bigg] \leq 2\listlen \nitems\;.
    \end{align*}
    Let $\paramSup$ denote the set such that $\cparam\in\paramSup$. By \citet[Lemma 6]{neu2022lifting}, for any $k$ and $\varepsilon>0$ we get
    $%
        \E\bigg[\sum_{t=1}^\timH \info_t(\cparam;\att_{t,\act_t(k)})\bigg]
        \le\log\left(
        \Ngaus_\varepsilon(\paramSup)\right)+\varepsilon \timH, 
    $ %
    where $\Ngaus_\varepsilon(\paramSup)$ is the $\varepsilon$-covering number of $\paramSup$ w.r.t $\ell_2$-norm. By \cref{asm:param} and the standard result on the covering number of the Euclidean ball in $\R^d$ we know $\ca{N}_\varepsilon(\paramSup)\le (\frac{2 \paramNorm}{\epsilon}+1)^d$. Thus, choosing $\varepsilon=1/\timH$ we get
    \begin{align*}
        \E\bigg[\sum_{t=1}^\timH \sum_{k=1}^\listlen \info_t(\cparam;\att_{t,\act_t(k)}) \bigg]
        &\le \listlen (d\log(2\paramNorm \timH +1) +1)\;.
    \end{align*} 
    Putting all the pieces together, we get the following bound
    \begin{align*}
        \breg(\timH)
        \le \listlen \sqrt{2 \nitems (d\log(2\paramNorm \timH +1) +1)} =\tilde{O}(\listlen\sqrt{d\nitems\timH})\;.
    \end{align*}
\end{proof}

\begin{lemma}[Bounding the Information Ratio]\label{lem:Info-Ratio}
    If $\att_{i,t}\in\{0,1\}$ for all $i,t$, then, $\infoRatio_{t,k}\leq 2 \sum_{i\in\items} \E_t[\posMeRew_{t}(\cntx_t,i)]$ for all $t\ge 1$ and $k\in \{1,\cdots,\listlen\}$.
\end{lemma}
\begin{proof}[Proof of \cref{lem:Info-Ratio}]

    Through Fenchel-Young inequality, we prove this using an idea similar to \citet[Lemma 5]{neu2022lifting}. First, for $p,q\in[0,1]$ let's remind the following definition of $d_{\text{KL}}$ for the Bernoulli distribution
    \begin{equation*}
        \binE(p || q) 
        \eqdef p \log\frac{p}{q} + (1-p)\log\frac{1-p}{1-q}\,,
    \end{equation*}
    where $0\log 0=0$ convention is used. Also, let $\posMeRew_{\act_t(k),t}(\cntx)\eqdef \E_t[\mrew_{\act_t(k)}(\cparam,\cntx)]$ be the posterior mean reward at round $t$ for item at position $k$ of the action, for parameter context tuple $(\cparam,\cntx)$. The Legendre–Fenchel conjugate of $\binE$ with respect to its first argument is defined for all $u \in \R$ as
    \begin{equation}\label{eq:LF}
        \binE^*(u\|q) \eqdef \sup_{p\in[0,1]} (pu - \binE(p\|q) ) = \log(1+q(e^u - 1)) 
        \;,
    \end{equation}
    where the second equality and the inequality follow from \cref{prop:dual_bound}. Let $\pi_{t,k}(i|\cntx_t)=\P(\act_t(k)=i|\cntx_t,\hist_t)$ be the probability of recommending item $i\in\items$ at position $k$ at round $t$. Given $\hist_t$ at round $t$, we let $\eqdist_t$ denote equality in distribution. Now define the pseudo-regret at step $t$ for position $k$ in the action as $\reg_{t,k}=\mrew_{\act^*(k)}(\cparam, \cntx_t)- \mrew_{\act_t(k)}(\cparam, \cntx_t)$, 
    then for any $\eta > 0$ and $k\in \{1,\cdots,\listlen\}$
    \begin{align*}
        \E_t[\reg_{t,k}] 
        &=
        \E_t\left[ %
        \mrew_{\act^*(k)}(\cparam, \cntx_t)-
        \mrew_{\act_t(k)}(\cparam, \cntx_t)
        \right]
        \\
        &=\E_t\left[ %
        \mrew_{\act_t(k)}(\tilde{\cparam}_t, \cntx_t)-
        \mrew_{\act_t(k)}(\cparam, \cntx_t)
        \right] \tag{As $(\cparam,\act^*_{t})\eqdist_t(\tilde{\cparam}_t,\act_{t})$}
        \\
        &= \E_t\left[  %
        \mrew_{\act_t(k)}(\tilde{\cparam}_t,\cntx_t)
        -\posMeRew_{\act_t(k),t}(\cntx_t)
        \right]
        \tag{By conditional independence of $\cparam$ and $\act_t$}
        \\
        &= \E_t\left[  %
        \bigg(\sum_{i\in\items}\indc(\act_t(k)=i) \frac{\eta \pi_{t,k}(i|\cntx_t) }{\eta \pi_{t,k}(i|\cntx_t)}
        \mrew_{i}(\tilde{\cparam}_t,\cntx_t)\bigg)
        -\posMeRew_{\act_t(k),t}(\cntx_t)
        \right]
        \\
        &\le \E_t\bigg[ %
        \bigg(\eta\sum_{i\in\items} \pi_{t,k}(i|\cntx_t)\bigg(\binE(\mrew_{i}(\tilde{\cparam}_t, \cntx_t)\|\posMeRew_{i,t}(\cntx_t))
        \\
        &\qquad\qquad\qquad\qquad\qquad
        +\binE^*\bigg(\frac{-\indc(\act_t(k)=i)}{\eta \pi_{t,k}(i|\cntx_t)}\bigg\| \posMeRew_{i,t}(\cntx_t)\bigg)\bigg)
        -\posMeRew_{\act_t(k),t}(\cntx_t)
        \bigg]\tag{By \cref{eq:LF} where $u\gets\frac{-\indc(\act_t(k)=i)}{\eta \pi_{t,k}(i|\cntx_t)}$, $p\gets\mrew_{i}(\tilde{\cparam}_t, \cntx_t)$, and $q\gets\posMeRew_{i,t}(\cntx_t)$}
        \\
        &\le \E_t\bigg[ %
        \bigg(\eta\sum_{i\in\items} \pi_{t,k}(i|\cntx_t)\bigg(\binE(\mrew_{i}(\tilde{\cparam}_t, \cntx_t)\|\posMeRew_{i,t}(\cntx_t))-\frac{\indc(\act_t(k)=i)}{\eta \pi_{t,k}(i|\cntx_t)}\posMeRew_{i,t}(\cntx_t)
        \\
        &\qquad\qquad\qquad\qquad\qquad +\frac{\indc(\act_t(k)=i)}{2\eta^2 \pi_{t,k}(i|\cntx_t)^2}\posMeRew_{i,t}(\cntx_t)\bigg)
        -\posMeRew_{\act_t(k),t}(\cntx_t)
        \bigg] \tag{By \cref{prop:dual_bound}}
        \\
        &= \E_t\bigg[ %
        \sum_{i\in\items} 
        \pi_{t,k}(i|\cntx_t) \binE(\mrew_{i}(\cparam, \cntx_t)\|\posMeRew_{i,t}(\cntx_t))
        +\frac{1}{2\eta} \posMeRew_{t}(\cntx_t,i)
        \bigg]
        \tag{By tower rule and $\E_t[\indc(\act_t(k)=i)]=\pi_{t,k}(i|\cntx_t)$ and $\cparam\eqdist_t\tilde{\cparam}_t$}
        \\ 
        &= \eta %
        \info_t(\cparam; \att_{\act_t(k),t}) +
        \frac{1}{2\eta}
        \sum_{i\in\items} \E_t[\posMeRew_{t}(\cntx_t,i)]
    \end{align*}
    Minimizing over $\eta$ we get
    \begin{align}
        \E_t[\reg_t] \leq \sqrt{2 \info_t(\cparam; \att_{\act_t(k),t}) \sum_{i\in\items} \E_t[\posMeRew_{t}(\cntx_t,i)]}\;,
    \end{align}
    which readily gives the result.
\end{proof}

\section{Alternative Logistic \ac{TS} Algorithm}\label{sec:log-alter}
In this section, we develop a logistic LTR algorithm that does not require solving equations like \cref{eq:Log-Param-estimate}. Therefore, this algorithm is more computationally attractive. We define $\bar{\obsind}=1$ if $\att_{i,t}=1$ and $\bar{\obsind}=-1$ if $\att_{i,t}=0$. \cref{alg:Log-LTR-newton} lays out the algorithm with step size $\alpha$. Employing standard analysis of online Newton step algorithms like an adaptation of \citet{hazan2007logarithmic,gentile2012multilabel} we can derive similar guarantees for our algorithm as \cref{alg:Log-LTR}. A close example of this is the algorithm in \citet{santara22a}.
\begin{algorithm}[bt]
   \caption{Logistic \ac{TS} for LTR using Newton Steps}
    \label{alg:Log-LTR-newton}
\begin{algorithmic}[1]
    \STATE {\bfseries Input:} Step size $\alpha$
    \STATE {\bf Initialize:} $c_1=1$, $\infMat_1=\listlen \iden$, and $\hat{\cparam}_1=0$
    \FOR{$t=1,\dots,\timH$}{
        \STATE Receive the current context $\cntx_t$.
        \STATE Sample $\tilde{\cparam}_{t} \sim \Ngaus(\hat{\cparam}_{c_t}, \infMat^{-1}_{c_t})$ \COMMENT{Posterior sample}
        \STATE $\act_t \in \argmax_{\substack{ a\subset\items\\|a|={\listlen}}}\sum_{i\in a} \cntx_{i,t}\T\tilde{\cparam}_t $\COMMENT{{ Recommend the top posterior samples} }
        \STATE Observe $\att_t$ and gather $\hist_{t+1}$ and $(\obs_{i,t})_{i\in\items}$.
        \FOR{$k\le \cmclick_t$}{
            \STATE $\infMat_{c_t+k-1} \gets \infMat_{c_t+k-2} + \obsind_{k,t}\cntx_{k,t}\cntx_{k,t}\T$
            \STATE $\hat{\cparam}_{c_t+k} \gets \hat{\cparam}_{c_t+k-1} + \alpha\linkf\left(-\bar{\obsind}_{k,t}\cparam_{c_t+k-1}\T \cntx_{i,t}\right) \bar{\obsind}_{k,t} \infMat_{c_t+k-1}^{-1}\cntx_{k,t}$
        }\ENDFOR
        \STATE $c_{t+1} = c_t +\cmclick_t$
    }\ENDFOR
\end{algorithmic}
\end{algorithm}

\section{\uppercase{Technical Tools}}

\begin{lemma}[Cauchy-Schwartz Inequality in Euclidean Vector Space]\label{lem:CS}
For all vectors $u$ and $v$ in $\R^d$ we know
    \begin{align*}
        \sum_{i=1}^d u_iv_i\leq \sqrt{\sum_{i=1}^d u_i^2}
        \sqrt{\sum_{i=1}^d v_i^2}\;.
    \end{align*}
\end{lemma}
This is a classical inequality that has several proofs \citep{steele2004cauchy}.

\begin{proposition}[\citet{neu2022lifting}, Proposition 1.]\label{prop:dual_bound}
	For any $u\le 0$ and $q\in[0,1]$:
	\begin{equation*}
	\binE^*(u\|q) \le q\left(u + \frac{u^2}{2}\right).
	\end{equation*}
\end{proposition}
\begin{proof}
	\begin{align}
	\binE^*(u\|q) = \log(1+q(e^u - 1)) \le q(e^u-1) \le q\left(u+\frac{u^2}{2}\right),\nonumber
	\end{align}
	where the first inequality is from $\log(1+x)\le x$ for any $x>-1$, and the second inequality is from $e^x\le 
	1+x+\frac{x^2}{2}$ for any $x\le 0$.
\end{proof}

\section{\uppercase{Further Experiments and Details}}\label{app:further-experiments}

\paragraph{Experimental setup configuration:} 
For the synthetic experiments, we used a machine with AMD EPYC 7B12, x86-64 processor with 48 cores and 200 GB memory, and each one took around 5 minutes to complete.

\paragraph{Dataset Experiment Details:}
For Web30K and Istella, we remove the features which have a standard deviation less than $10^{-6}$ after normalizing across the datasets.

\paragraph{Further experiments:}
\cref{fig:web30k-cut} shows a truncated version of \cref{fig:web30k} for a closer look into the early rounds of the experiment. As we can see, \algLogTSLTR outperforms the other algorithms.

\cref{fig:web30k-all} shows that \algGTS\ has a competitive performance compared to \algTS.

\begin{figure}[tb]
\centering
    \begin{minipage}[t]{.32\textwidth}
    \includegraphics[width=\textwidth]{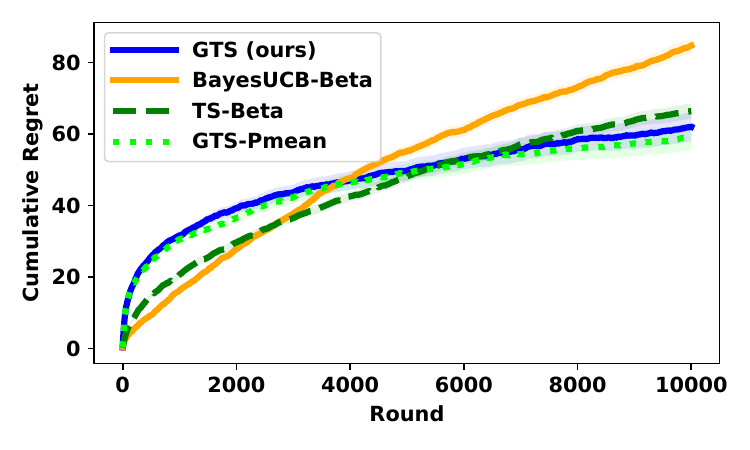}
    \caption{Synthetic non-contextual setting with a subset of algorithms to highlight the differences.}
    \label{fig:stand-subset} %
    \end{minipage}
\hfill
    \begin{minipage}[t]{.32\textwidth}
    \includegraphics[width=\textwidth]{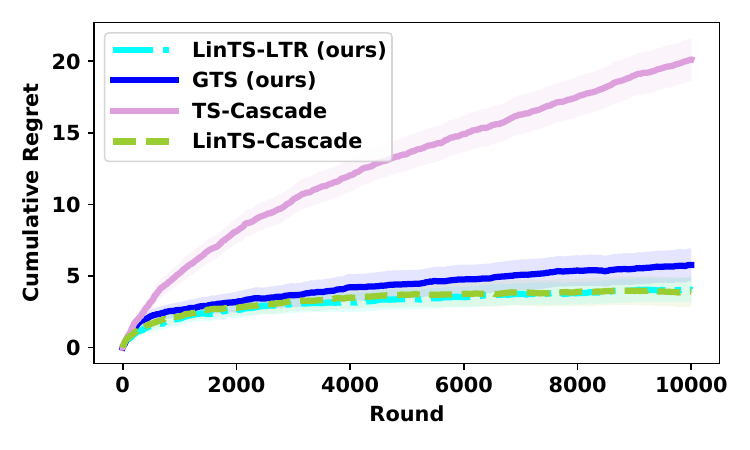}
    \caption{Linear model experiment same as \cref{fig:lin} but some of the algorithms were removed to focus on the difference between the high performing ones.}
    \label{fig:lin-subset}
    \end{minipage}
\hfill
    \begin{minipage}[t]{.32\textwidth}
     \includegraphics[width=\textwidth]{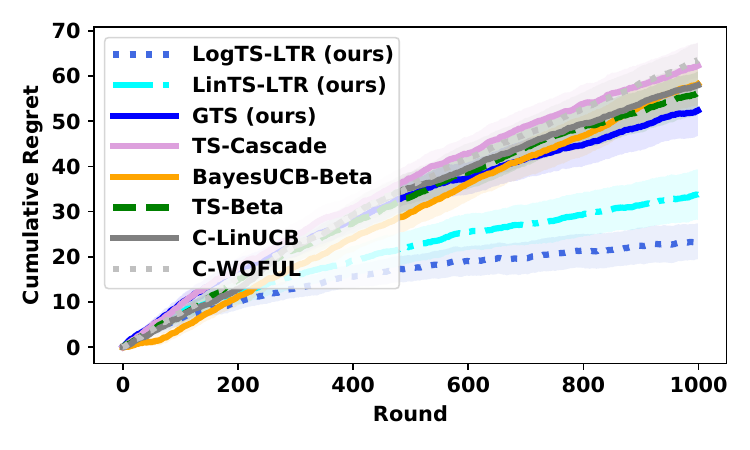}
    \caption{The logistic model experiment showing the first 1000 steps.}
    \label{fig:log-cut}
    \end{minipage}
\end{figure}

\begin{figure}[tb]
\centering
    \begin{minipage}[t]{.45\textwidth}
     \includegraphics[width=\textwidth]{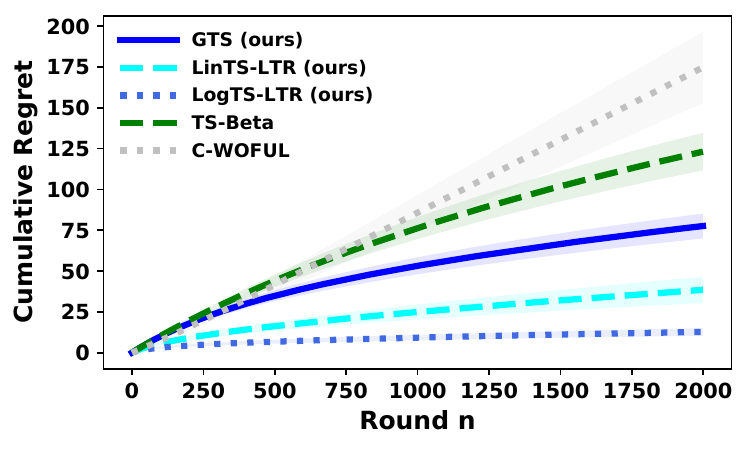}
     \caption{Web30k experiment truncated at 2000 steps.}
     \label{fig:web30k-cut}
    \end{minipage}
\hfill
    \begin{minipage}[t]{.45\textwidth}
     \includegraphics[width=\textwidth]{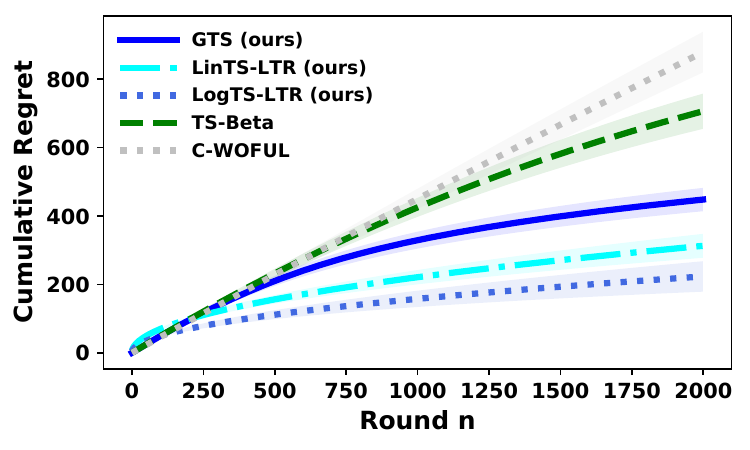}
     \caption{Istella experiment truncated at 2000 steps.}
     \label{fig:Istella-cut}
    \end{minipage}
\hfill
    \begin{minipage}[t]{.45\textwidth}
     \includegraphics[width=\textwidth]{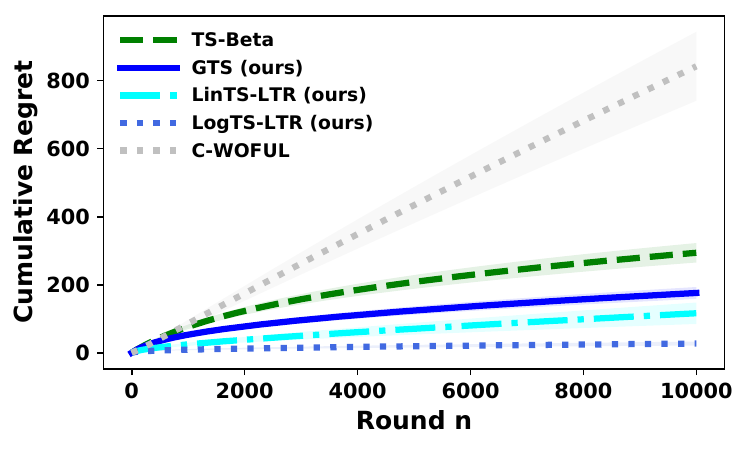}
     \caption{The Web30k experiment with all the algorithms.}
     \label{fig:web30k-all}
    \end{minipage}
\hfill
    \begin{minipage}[t]{.45\textwidth}
     \includegraphics[width=\textwidth]{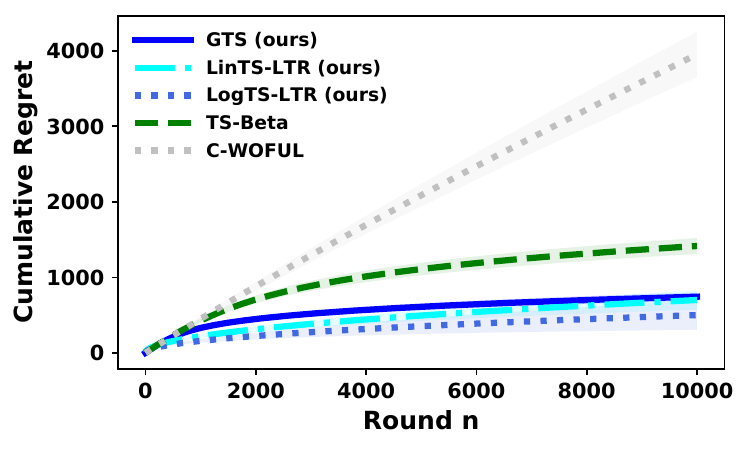}
     \caption{The Istella experiment with all the algorithms.}
     \label{fig:Istella-all}
    \end{minipage}
\end{figure}

\cref{fig:prior-standard-cut,fig:prior-standard} demonstrate the prior initialization where \texttt{BetaMean} is the mean of the (true) Beta prior, and \texttt{BetaVar} is its variance. The legend denotes the [prior mean, variance] for the Gaussian prior of \algGPTS. We can observe that the prior does not have a huge impact on our \algGPTS algorithm (another confirmation to its robustness w.r.t prior misspecification). It seems, however, the correct prior ([\texttt{BetaMean, BetaVar}]) shows slightly better performance in the early stages.

\begin{figure}[t!]
\centering
    \begin{minipage}[t]{.45\textwidth}
    \centering
    \includegraphics[width=\textwidth]{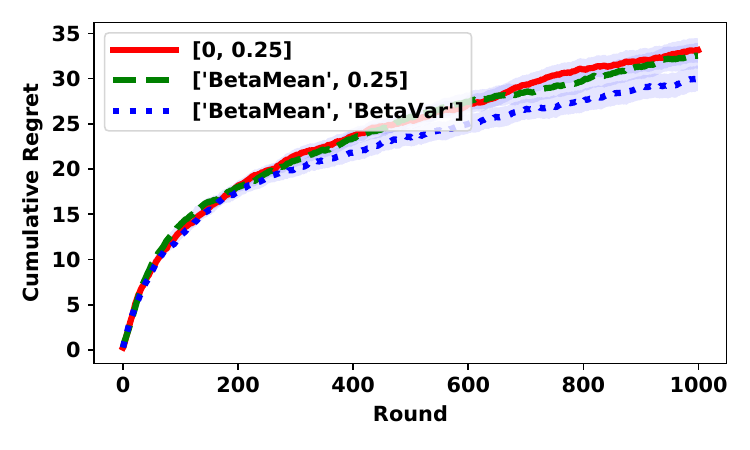}
    \caption{The effect of prior initialization on \algGPTS in a Beta environment, $\timH=1000$.}
    \label{fig:prior-standard-cut}
    \end{minipage}
\hfill
    \begin{minipage}[t]{.45\textwidth}
    \centering
    \includegraphics[width=\textwidth]{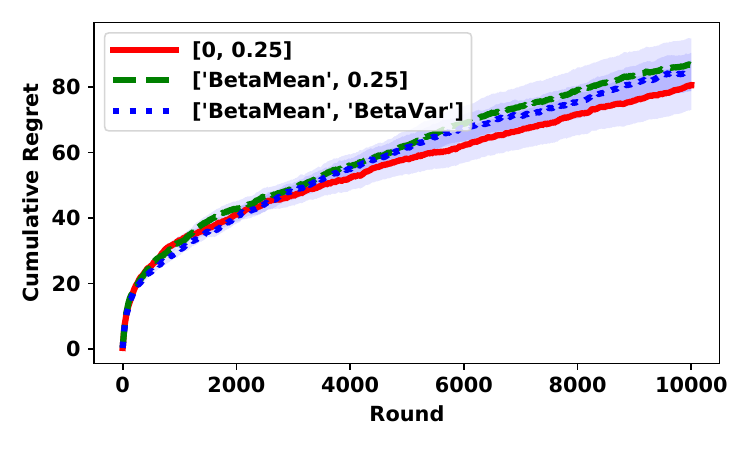}
    \caption{The effect of prior initialization on \algGPTS in a Beta environment, $\timH=10000$.}
    \label{fig:prior-standard}
    \end{minipage}
\end{figure}
\vspace{200pt}
\;

\end{document}